\documentclass[onefignum,onetabnum,onealgnum,onethmnum]{siamonline190516}
\usepackage{lipsum}

\usepackage{endnotes}

\usepackage{titlesec}
\usepackage{titletoc}

\titleclass{\task}{straight}[\section]
\newcounter{task}
\renewcommand{\thetask}{\arabic{task}}
\titleformat{\task}[hang]
    {\normalfont\LARGE\bfseries}{Task \thetask:}{1em}{}

\titleformat*{\task}{\color{header1}\bfseries}

\titlecontents{task}
              [3.8em] 
              {}
              {\contentslabel{2.3em}}
              {\hspace*{-2.3em}}
              {\titlerule*[1pc]{.}\contentspage}

\titlespacing*{\section}{0ex}{1ex}{1ex}
\titlespacing*{\subsection}{0ex}{1ex}{1ex}
\titlespacing*{\paragraph}{0ex}{1ex}{1ex}
\titlespacing*{\subparagraph}{0pt}{1ex}{1ex}
\titlespacing*{\task}{0em}{1ex}{1ex}

\usepackage[sort&compress,comma,square,numbers]{natbib}

\usepackage{amsfonts,amsmath,amssymb}
\usepackage{bbm}

\usepackage{multicol}
\usepackage{enumitem}

\usepackage{hhline}

\usepackage{comment}
\usepackage{todonotes}
\RequirePackage{ulem}

\usepackage{booktabs}
\usepackage{titlesec}
\usepackage{fancyhdr}
\usepackage{fullpage}

\RequirePackage{titlesec}
\RequirePackage[font={small},{singlelinecheck=false}]{caption}
\usepackage[T1]{fontenc}
\usepackage[scaled]{helvet}
\usepackage[scaled=1.10, med]{zlmtt}

\usepackage{algorithm}
\usepackage{algpseudocode}

\newtheorem{lem}{Lemma}

\renewcommand{\algorithmicrequire}{\textbf{Input:}}
\renewcommand{\algorithmicensure}{\textbf{Output:}}

\newcommand{\argmax}{\operatornamewithlimits{argmax}}

\newcommand{\T}{^{\ensuremath{\mathsf{T}}}}           
\providecommand{\mc}[1]{\mathcal{#1}}
\providecommand{\mb}[1]{\boldsymbol{#1}}

\providecommand{\mt}[1]{\widetilde{#1}}

\providecommand{\mv}[1]{\vec{#1}}
\providecommand{\mh}[1]{\hat{#1}}

\newcommand{\Real}{\mathbb{R}}

\newcommand{\II}{\mathbb{I}}           

\newtheorem{assumption}{Assumption}

\newcommand{\E}{\ensuremath{\mathbb{E}}}
\newcommand{\I}{\ensuremath{\mathbbm{1}}}

\begin{document}

\title{Random Forests for Adaptive Nearest Neighbor Estimation of Information-Theoretic Quantities}
\author{Ronan Perry$^1$, Ronak Mehta$^1$, Richard Guo$^1$, Eva Yezerets$^1$, Jes\'us Arroyo$^1$, Mike Powell$^1$, Hayden Helm$^1$, Cencheng Shen$^1$, and Joshua T. Vogelstein$^{1,2}$\thanks{Corresponding author: \email{jovo@jhu.edu}; $^1$Johns Hopkins University, $^2$Progressive Learning}}

\maketitle
\thispagestyle{empty}

\noindent
\setcounter{tocdepth}{3}
\setcounter{secnumdepth}{3}

\vspace{-15pt}

\pagenumbering{arabic}
\setcounter{page}{1}

\setcounter{theorem}{0}
\setcounter{lem}{0}
\setcounter{assumption}{0}

\begin{abstract}
    Information-theoretic quantities, such as conditional entropy and mutual information, are critical data summaries for quantifying uncertainty.
    Current widely used approaches for computing such quantities rely on nearest neighbor methods and exhibit both strong performance and theoretical guarantees in certain simple scenarios.
    However, existing approaches fail in high-dimensional settings and when different features are measured on different scales.
    We propose decision forest-based adaptive nearest neighbor estimators and show that they are able to effectively estimate posterior probabilities, conditional entropies, and mutual information even in the aforementioned settings.
    We provide an extensive study of efficacy for classification and posterior probability estimation, and prove certain forest-based approaches to be consistent estimators of the true posteriors and derived information-theoretic quantities under certain assumptions.
    In a real-world connectome application, we quantify the information contained in various biological and learned features about neuron cell type in the Drosophila larva mushroom body.
\end{abstract}


\section{Introduction}

Uncertainty quantification is a fundamental desiderata of statistical inference and data science.  In supervised learning settings it is common to quantify uncertainty with either conditional entropy or mutual information (MI). 
Suppose we are given a pair of random variables $(X, Y)$, where $X$ is $d$-dimensional vector-valued and $Y$ is a categorical variable of interest. Conditional entropy $H(Y|X)$ measures the uncertainty in $Y$ on average given $X$. On the other hand, mutual information quantifies the shared information between $X$ and $Y$.
Although both quantities are readily estimated when $X$ and $Y$ are low-dimensional and ``nicely'' distributed, an important problem arises in measuring these quantities from higher-dimensional data in a nonparametric fashion \citep{mines}. A popular family of mutual information estimators \citep{ksg, mixedksg} are derived from the $k$-nearest neighbor ($k$-NN) classifier. However, they are poorly suited in settings where the $L2$ distance, $d(X_1, X_2) = \sqrt{\sum_{k=1}^d(X_{1k} - X_{2k})^2}$, is uninformative, such as high-dimensional noisy settings and when features exist on different scales, both cases which would require potentially nontrivial and unstable preprocessing.

We address this limitation by proposing decision forest methods for estimating conditional entropy under the framework that $X$ is any  $d$-dimensional random vector and $Y$ is  categorical. Decision trees \citep{statlearning}, ensembled to compose a decision forest, are effectively adaptive nearest neighbor estimators in that they partition the feature space into regions in a process invariant to monotone transformations of the features and which has the capacity to ignore unimportant features \citep{lin_random_2006, hastie_discriminant_1996, balsubramani_adaptive_2019}. 
Classification decision forests make predictions by learning estimates of the posterior probabilities $P(Y | X)$. Because $Y$ is categorical, one can easily compute the maximum-likelihood estimate of $H(Y | X)$ from the posteriors \cite{mukherjee_ccmi_2020}. Besides the canonical Random Forest \citep{Breiman2001} built from Classification and Regression Trees \citep{Breiman1984-aq}, a variety of modified decision forest methods have been developed to better estimate posterior probabilities (i.e., better \textit{calibrated}). These include the Platt Scaling \citep{platt_probabilistic_1999} and Isotonic Regression \citep{isotonic} post-hoc calibrations of decision forest posteriors to better reflect estimates on a hold-out subset of the  data. Another variant developed with asymptotic theory in mind fits the posteriors of a tree using a hold-out, \textit{honest} subsample independent of the data used to learn the tree \citep{denil14, Wager2018}. Despite the success of these honest forests in a variety of regression settings, we are unaware of any study of honesty in the case of classification. Although posterior estimation relies on conditional mean estimation just as in a regression setting, it is a harder task in that one only has access to the categorical responses and never actually observes the continuous posteriors which are to be estimated.

\paragraph{Contributions:} We summarize our contributions as an outline of the paper.
\begin{itemize}
    \item We provide the first large-scale comparison of the accuracy and calibration of honest decision forests.
    \item In simulated examples, we demonstrate specific regimes where honest forests are better calibrated than other forest methods.
    \item We prove theoretical consistency results of honest forest posterior probability estimates and mutual information (MI) estimates.
    \item We illustrate three ways to compute MI using decision forests and demonstrate through simulations that not only do honest forests produce the most robust, accurate estimates across settings, but also that all forest methods show benefits over the $k$-NN-based methods in certain settings.
    \item We measure the MI between various \textit{Drosophila} neuron properties and the neuron cell type and show the results provide a meaningful understanding of the features.
\end{itemize}

\section{Background}
We begin by providing an overview of conditional entropy and mutual information, as well as a review of methods aimed at estimating these quantities. After motivating decision forests as alternatives to the popular $k$-Nearest Neighbor methods, we provide an overview of four different decision forest variants with experiments highlighting how they differ in terms of posterior probability estimation; these posterior probabilities will serve as the basis for the conditional entropy and mutual information estimates in later sections.

\subsection{Mutual Information}

Suppose we are given two random variables $X$ and $Y$ with support sets $\mathcal{X} \subseteq \Real^d$ and $\mathcal{Y} := [K] = \{1, ..., K\}$, respectively, for positive integers $d$ and $K$. Let $x$, $y$ denote specific values that the random variables take on, with $p(x)$ and $p(y)$ denoting the probabilities $P(X=x)$ and $P(Y=y)$, respectively. While $P(Y = y)$ measures the uncertainty surrounding the value $y$, the Shannon entropy,
$$H(Y) = -\sum_{y \in \mathcal{Y}} p(y) \log p(y),$$
measures the total uncertainty of the random variable $Y$. In the presence of information on the random variable $X$, the conditional probability, or posterior, $P(Y = y | X = x)$ reflects the uncertainty of $y$ given $X = x$. Analogously, we can write the entropy of $Y$ given $X = x$ as
$$H(Y | X = x) = -\sum_{y \in \mathcal{Y}} p(y | x) \log p(y | x)$$
and the conditional entropy as
$$ H(Y|X) = \sum_{x \in \mathcal{X}}{p(x)H(Y|X = x)} = -\sum_{x \in \mathcal{X}}{p(x)\sum_{y \in \mathcal{Y}}{p(y|x)\log{p(y|x)}}}.$$ The conditional entropy represents the expected uncertainty in $Y$ having observed $X$. In the case of a continuous random variable, the sum over the corresponding support is replaced with an integral, and probability mass functions are replaced by densities; this is technically the differential entropy, a limiting case of conventional entropy with some minor differences, but we will refer to both cases as entropy throughout the rest of this paper. Additionally, we will use the natural logarithm for all entropy calculations in this paper.

Mutual information $I(X; Y)$ in turn measures the mutual dependence of $X$ and $Y$, and can be computed from conditional entropy symmetrically in the equalities
$$I(X; Y) = H(Y) - H(Y|X) = H(X) - H(X|Y).$$
Mutual information has many appealing properties, such as symmetry and non-negativity, and is widely used in data science applications such as feature selection and independence testing \citep{mixedksg}. For instance, $I(X; Y) = 0$ if and only if $X$ and $Y$ are independent.

\subsection{Related Work}

A common approach to estimating mutual information relies on breaking it into a sum of terms and estimating each term separately. One such approach uses the \textit{3-H} principle \citep{mixedksg}, written as
$I(X; Y) = H(Y) + H(X) - H(X, Y)$,
where $H(X, Y) = -\sum_{x \in \mathcal{X}} \sum_{y \in \mathcal{Y}} p(x,y) \log p(x,y)$ is the Shannon entropy of the pair $(X, Y)$. Examples using this approach include kernel density estimators and ensembles of $k$-NN estimators \citep{Beirlant1997, Leonenko08estimationof, berrett2019, sricharan}. One method in particular, the KSG estimator, improves $k$-NN estimates via heuristics and is popular because of its excellent empirical performance \citep{ksg}. Other approaches include binning~\citep{binning} and von Mises estimators~\citep{vonmises}. Unfortunately, many modern datasets contain mixtures of continuous inputs $X$ and discrete outputs $Y$. In these cases, while individual entropies $H(X)$ and $H(Y)$ are well-defined, $H(Y, X)$ is either ill-defined or not easily estimated, thus rendering the \textit{3-H} approach intractable~\citep{mixedksg}. A recent approach, referred to as Mixed KSG \citep{mixedksg}, modifies the KSG estimator to improve its performance in various settings, including mixed continuous and categorical inputs. 

Computing both mutual information and conditional entropy becomes difficult in higher dimensional data. Numerical summations or integration become computationally intractable, and nonparametric methods (for example, $k$-nearest neighbor, kernel density estimates, binning, Edgeworth approximation, likelihood ratio estimators) typically do not scale well with increasing dimensions \citep{mines, gaoetal}. While the recently proposed neural network approach MINE \citep{mines} addresses high-dimensional data, the method does not return an estimate of the posterior $p(y \mid x)$, which is useful in uncertainty quantification \citep{sarawgi_uncertainty-aware_2021}, and has been found difficult to properly tune correctly even in toy cases \citep{mukherjee_ccmi_2020}. Such deep learning approaches additionally make assumptions that are difficult to check or assume in practice in order to guarantee good estimation of the target quantity, either requiring that the network of choice be expressive enough so that the Donsker-Varadhan representation can be used \citep{mukherjee_ccmi_2020}, or limiting the distribution of learned weights \citep{deep1}.

Another approach is to use the decomposition $I(X; Y) = H(Y) - H(Y|X)$ \citep{mukherjee_ccmi_2020}. The entropy $H(Y)$ of the discrete random variable $Y$ is easily estimatable from the empirical frequencies. The conditional entropy $H(Y|X)$ meanwhile is a function of the data only through the posterior $p(y \mid x)$ and integration over the feature space. So, an estimate of the posterior with an estimate of the density is sufficient to estimate mutual information. Through this approach, any algorithm which can estimate posterior probabilities is applicable \citep{mukherjee_ccmi_2020}.

\section{Posterior Estimation for Uncertainty Quantification}

In order to compute mutual information and other entropy quantities via posterior probabilities, we must first establish and understand methods that estimate posterior probabilities.

\subsection{Decision Forest Methods for Posterior Estimation} \label{sec:cart}

Decision forests are powerful classification algorithms that estimate posterior probabilities. They function as adaptive nearest neighbor methods and so are natural alternatives to the limited $k$-nearest neighbor approaches underlying the KSG estimators. We elaborate below on the qualification of the random forest algorithm and two of its variants.

\subsubsection{Random Forests}
Random forest (RF) is a robust, powerful algorithm that leverages ensembles of decision trees for classification tasks \citep{Breiman2001} by computing posterior probabilities. In a study of over 100 classification problems of mostly tabular data sets, F\'ernandez-Delgado et al. \citep{Delgado14} showed that random forests have the overall best performance when compared to 178 other classifiers. Furthermore, random forests are highly scalable; efficient implementations can build a forest of 100 trees from 110 Gigabyte data ($n$ = 10,000,000, $d$ = 1000) in little more than an hour \citep{scalablerf}.

Random forest classifiers are instances of a bagging classifiers in which many decision trees are learned on bootstrapped subsets of the data and then ensembled together. A decision tree first learns a partition of feature space and then learns constant functions within each region of the feature space to perform classification or regression. Precisely, $ \mc{L} $ is called a partition of feature space $ \mc{X} $, if for every $ L, L' \in \mc{L}$ with $L \neq L'$, $L \cap L' = \varnothing$ and $ \bigcup_{L \in \mc{L}} L = \mc{X} $. This $ \mc{L} $ is learned on data $\{X_i, Y_i\}_{i=1}^n$ by recursively splitting a randomly selected subsample along a single dimension of the input data based on an impurity measure \citep{Breiman2001}, such as Gini impurity or entropy. Decision trees are binary trees grown by recursively adding two nodes at each split until a node reaches a certain criterion (for example, a minimum number of samples). These unsplit nodes are called ``leaf nodes''. To better decorrelated the trees, typically only a random subset of dimensions of $X$ are considered at each split node.

Given a partition $ \mc{L} $, let $L(x)$ be the part of $ \mc{L} $ to which $x \in \mc{X}$ belongs. Letting $\I[\cdot]$ be the indicator function, a possible predictor function $\hat{g}$ for classification is $\hat{g}(x) = \argmax_{y \in \mc{Y}} \sum_{i=1}^n \I[Y_i = y, X_i \in L(x)]$. This is the plurality vote among the training data in the leaf node of $x$. For a regression task, a corresponding predictor $\hat{\mu}$ is $\hat{\mu}(x) = \sum_{i = 1}^n Y_i \cdot \I[X_i \in L(x)] / \sum_{i = 1}^n \I[X_i \in L(x)]$, which is the average $Y$ value for the training data in $L(x)$. For a forest, $B$ trees are learned on randomly subsampled points. These points used in tree construction are called the `in-bag' samples for that tree, while those that are left out are called `out-of-bag' (\textit{oob}) samples.

\subsubsection{Calibrated Forests}

A common problem for classification algorithms in general is ensuring that predicted posteriors are accurate (calibrated). For a perfectly calibrated classifier, the predicted posterior equals the true, but unknown, posterior. Although RF is powerful, its posteriors are not necessarily well calibrated \citep{zadrozny_obtaining_2001, niculescu-mizil_predicting_2005}. Two popular and well-studied calibration methods applicable to classifiers such as RF are Platt Scaling \citep{platt_probabilistic_1999} and Isotonic Regression \citep{isotonic}.

Platt Scaling takes in a classifier $f$ fit on part of the data and then fits a logistic function
$$P(Y = y \mid f(y \mid x)) = \frac{1}{1 + \exp(af(y \mid x) + b)}$$
on top of the classifier posteriors $f(y \mid x)$, where $a$ and $b$ are optimized to minimize the log loss on the remaining, held-out data. Isotonic Regression similarly takes in a classifier $f$ fit on part of the data but instead fits the model
$$P(Y = 1 \mid f(y \mid x)) = m(f(y \mid x)) + \varepsilon,$$
where $m$ is a monotonically increasing function that minimizes the mean-squared error on the remaining, held-out data. Isotonic Regression is more flexible, but it suffers in low sample size settings compared to Platt Scaling.

Platt Scaling Random Forests (SigRF) and Isotonic Regression Random Forests (IRF) have been tested extensively \citep{zadrozny_obtaining_2001, niculescu-mizil_predicting_2005}, and both empirically improve posterior calibration over RF. The Python implementations of SigRF, IRF, and RF used throughout this paper are publicly available from the \texttt{scikit-learn} Python package \citep{scikit-learn}.

\subsubsection{Honest Forests}
In conventional decision forest algorithms, data are typically split to maximize purity within child nodes, so the posteriors estimated in each leaf tend to be biased toward certainty. \textit{Honesty} helps in bounding the bias of tree-based estimates \citep{denil14,Athey2019-uw,Wager2018}. \citet{Wager2018} describe honest trees as those for which any particular training example $(X_i, Y_i)$ is used to either partition feature space or to estimate the quantity of interest, but not both.
One way to achieve this property is to split the observed samples into two sets for each tree: one set for learning the partitions of feature space $\mathcal{X}$ and one set to establish the plurality or average within each leaf node. 
We refer to them as the  ``partition" set $\mathcal{D}^{\text{P}}$ and ``voting" set $\mathcal{D}^{\text{V}}$, respectively.

For example, say we wish to estimate the conditional mean function $\mu(x) = \mathbb{E}[Y \mid X = x]$ in a single tree. Let $\mathcal{L}$ be a partition of feature space $\mc{X}$, as described in Section \ref{sec:cart}. Letting $m < n$, such an $\mc{L}$ can be learned via a decision tree with $\mathcal{D}^{\text{P}} = \{(X_1,Y_1),...,(X_m,Y_m)\}$. This leaves $\mathcal{D}^{\text{V}} = \{(X_{m+1},Y_{m+1}),...,(X_n,Y_n)\}$. Letting $L(x)$ be the part of $\mathcal{L}$ to which $x$ belongs, the conditional mean estimate can be
$$\hat{\mu}(x) = \frac{\sum_{\mathcal{D}^{\text{V}}} Y_i \cdot \I[X_i \in L(x)]}{\sum_{\mathcal{D}^{\text{V}}} \I[X_i \in L(x)]} = \frac{\sum_{i = m + 1}^n Y_i \cdot \I[X_i \in L(x)]}{\sum_{i = m + 1}^n \I[X_i \in L(x)]}.$$

\noindent
When learning the entire forest, each tree uses a different random partition of the samples into partition and voting subsets. This approach is data efficient in that each sample will always be used for either learning the tree structure or fitting the posteriors in each tree, and with high probability each sample will be used for both across all trees in the forest.


While honesty has been studied in the context of regression estimates \citep{Wager2018, Athey2019-uw}, here we provide empirical results and an extension of their consistency results to posterior prediction for classification tasks, as well as functions of the posterior probabilities. We term the honest Random Forest for posterior estimation as Uncertainty Forest (UF) because it is specifically estimating the uncertainty of a categorical label. An open source Python implementation available at  \url{https://neurodata.io/code/} has been developed and is used throughout this paper.

\subsection{Asymptotic Consistency of Honest Posterior Estimates}\label{sec:consistency_posteriors}

Under mild conditions, UF provides a consistent estimate of posterior probabilities. This will be useful later as well for results on the consistent estimates of conditional entropy and mutual information. The basis of this result comes from \citet{Athey2019-uw} who establish a class of forests whose estimates are asymptotically consistent under mild distributional assumptions, specifications of the forest algorithm, and requirements on the quantity being estimated. We follow their conditions to prove that UF posterior probability estimates are consistent.

We assume that $\mathcal{Y}$ is discrete ($\mathcal{Y} = [K]$), and that $\mathcal{X} = \Real^d$. Denote the finite sample conditional entropy estimate as $\hat{H}_n(Y \mid X)$ and mutual information estimate as $\hat{I}_n(X; Y)$, now indexed by the sample size $n$ for explicitness. Consider the following specification and assumptions.

\paragraph{Specification 1}\label{spec}
Directly from \citet{Athey2019-uw}, we require that (1) all trees are invariant to the ordering of the training examples, (2) at each split each child node receives a nonzero fraction of the samples in the parent node and splits with nonzero probability on each feature (this can be done by randomly choosing the number of candidate features to split on at each node \citep{denil14}), (3) each tree is learned on a subsample of $\mathcal{D}_n$ with size $s_n$ chosen such that $s_n \rightarrow \infty$ and $\frac{s_n}{n} \rightarrow 0$, and (4) each tree has $|\mathcal{D}^{\text{V}}| = \gamma s_n$ honest samples for $0 < \gamma < 1$.

\begin{assumption}
    \label{assmp:support}
    Suppose that $X$ is supported on $\Real^d$ and has a density which is non-zero almost everywhere. This is equivalent to having $X$ be uniformly distributed in $[0, 1]^d$ due to the monotone invariance of trees \citep{denil14}.
\end{assumption}
\begin{assumption}
    \label{assmp:lipshitz}
    Suppose that for each $y \in [K]$ the conditional probability $p(y \mid x)$ is Lipshitz continuous on $\Real^d$.
\end{assumption}

With the proof given in Appendix \ref{sec:proof}, we have the following lemma.

\begin{lem}[Consistency of the Posterior Probabilities]
    Given Assumptions 1-2 and UF built to Specification 1, for all $y \in [K]$ and all $x \in \Real^d$, the estimate $\hat{p}_n(y \mid x) \overset{P}{\rightarrow} p(y \mid x)$ as $n \rightarrow \infty$.
    \label{thm:prob}
\end{lem}
This yields the consistent estimation of posterior probabilities in the classification setting. The required Assumptions and Specification meet those of \citet{Athey2019-uw} and the posterior probability is an appropriate conditional expectation quantity.

\subsection{Empirical Results: Random Forest Posterior Estimation}

We now provide empirical results illustrating how the four aforementioned forest algorithms compare across various simulated and real settings.

\subsubsection{Simulation: Uncertain Posteriors}\label{sim:overlapping_gaussians}
An illustration of how RF, SigRF, IRF, and UF all differ in posterior estimation is provided in Figure \ref{fig:overlapping_gaussians}. Inspired by a \texttt{scikit-learn} tutorial \citep{scikit-learn}, data from two classes are sampled where each class is a mixture of two isotropic multivariate Gaussian distributions. Let $\mathcal{N}(x; \mu)$ denote the density of a Gaussian distribution with mean $\mu$ and identity covariance. Then the class-conditional density is $f(X=x \mid Y=k) = \frac13 \mathcal{N}(x; [0, 0]^T) + \frac23 \mathcal{N}(x; \mu_k)$ where $k \in \{-1, 1\}$, $P(Y = k) = \frac12$, and $\mu_k = [5k, 5k]^T$. Each forest was composed of $100$ trees, and SigRF and IRF used $5$-fold internal cross validation on $20$ trees per fold to calibrate the posteriors. UF used half of the training samples per tree for learning the structure and half for setting the posteriors.

Although there exists high uncertainty where the classes share a Gaussian density centered at $\mu = [0, 0]^T$, the RF predictions are overconfident. SigRF is able to mitigate this slighty in some regions at the cost of other regions, but IRF and UF do much better jobs at calibrating overall.

\begin{figure}[htb!]
  \centering
  \includegraphics[width=\linewidth]{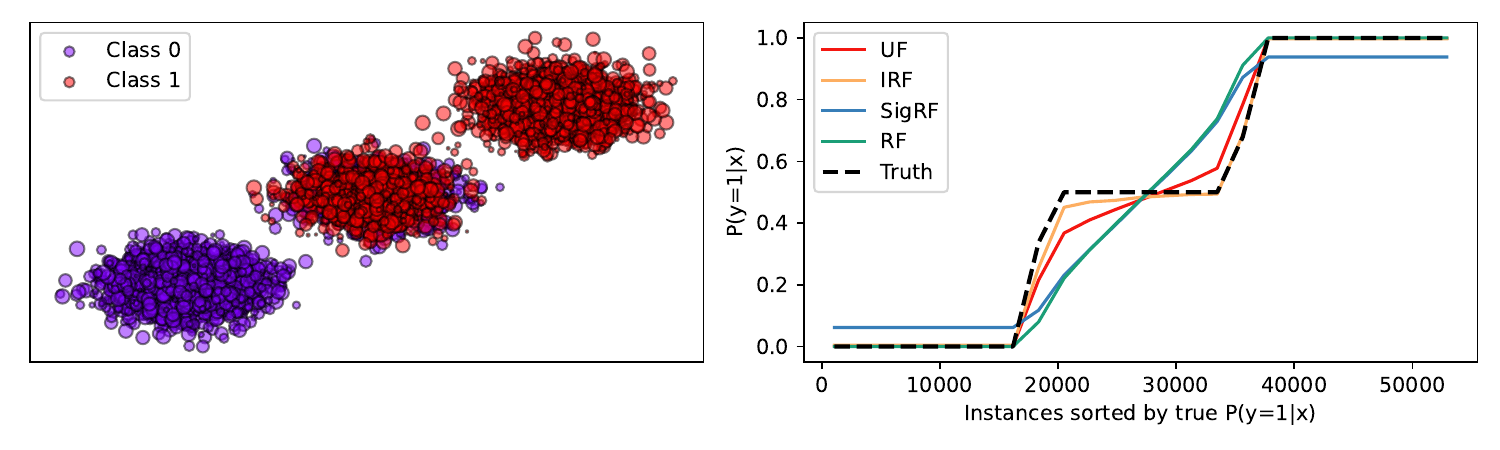}
  \caption{(\textbf{Left}) Simulated data. (\textbf{Right}) Although RF is overconfident in its predicted posteriors, SigRF, IRF, and UF all calibrate the posteriors to some extent. The true sample posteriors are plotted as a dashed line. $6$,$000$ samples were used for training, and posteriors were evaluated on $54$,$000$ samples from the same distribution.}
  \label{fig:overlapping_gaussians}
\end{figure}

\subsubsection{Simulation: Steep Posteriors}\label{sim:steep_posteriors}
A common problem for a variety of machine learning algorithms is prediction at the boundary of the feature space. \citet{Wager2018} showed how honest regression forests perform better when there is a sharp transition point in the regression estimate; inspired by their work, we show how honesty improves classification when the posterior changes abruptly. Let $X \sim Unif[0, 1]^d$ be a $d$-dimensional random variable distributed uniformly in a $d$-dimensional hypercube. Let $\mathcal{Y} = \{0, 1\}$ and $P(Y = 1 | X = x) = \prod_{j=1}^2[1 + \exp(-\alpha(x_j-1/2))]^{-1}$ where $x_j$ is the $j$th dimension of $x$ and $\alpha$ is a hyperparameter controlling the rate at which the posterior changes near the decision boundary. When $\alpha=0.5$, the posterior is flat everywhere and there is no possible predictive power, but as $\alpha$ increases, the class-specific posteriors approaches certainty. Note that only the first two dimensions of $X$ are informative while the rest are  uninformative noise features.

Each of the forests was fit to $5$,$000$ samples from the above joint distribution given a varying $\alpha$ and fixed dimension $d=4$. The calibration error of a trained forest's predicted $k$-class posterior vector $p(x)$ at a point $x$ was evaluated using the Hellinger distance $\frac{1}{\sqrt{2}}\sum_{j=1}^k (\sqrt{p_j(x)} - \sqrt{q_j(x)})^2$ to the known true posterior $q(x)$. The expected Hellinger loss of each forest was approximated by the average error of each forest on $40$,$000$ new samples whose first two (informative) dimensions were deterministically spaced equally across a $50\times50$ grid and remaining $d-2$ dimensions were sampled uniformly. Each of the decision forest methods were composed of $500$ trees, $5$ folds of $100$ trees for IRF and SigRF, and used all of the features in each split due to the sparsity of the signal. Here, UF used $63\%$ of the samples to learn the tree structure to better equate to the number of unique samples used during bootstrapping, although iterations using other resampling percentages indicated that the results are insensitive to this value in this scenario.

As shown in the leftmost plot of Figure \ref{fig:steep_posteriors}, once $\alpha$ is slightly larger than $1$, UF dominates the other methods. With increasing sharpness $\alpha$, the non-UF methods degrade in their calibration. On the right side of Figure \ref{fig:steep_posteriors}, we see in the upper left heatmap how at $\alpha = 1$, the posteriors appear close to uniform across the first two dimensions whereas at $\alpha=12$, in the lower left heatmap, the posteriors transition more sharply. Each of the subsequent columns show how the predicted posteriors for each method and $\alpha$ differ from the true posteriors. A red value means that the forest is underconfident while a blue value means that is overconfident. We see pronounced effects at the boundary of the feature space and at the decision boundary, although UF mitigates these artifacts best. Additional heatmaps and plots may be found in Appendix \ref{sec:supp_steep}.

\begin{figure}[htb!]
    \centering
    \includegraphics[width=0.42\linewidth]{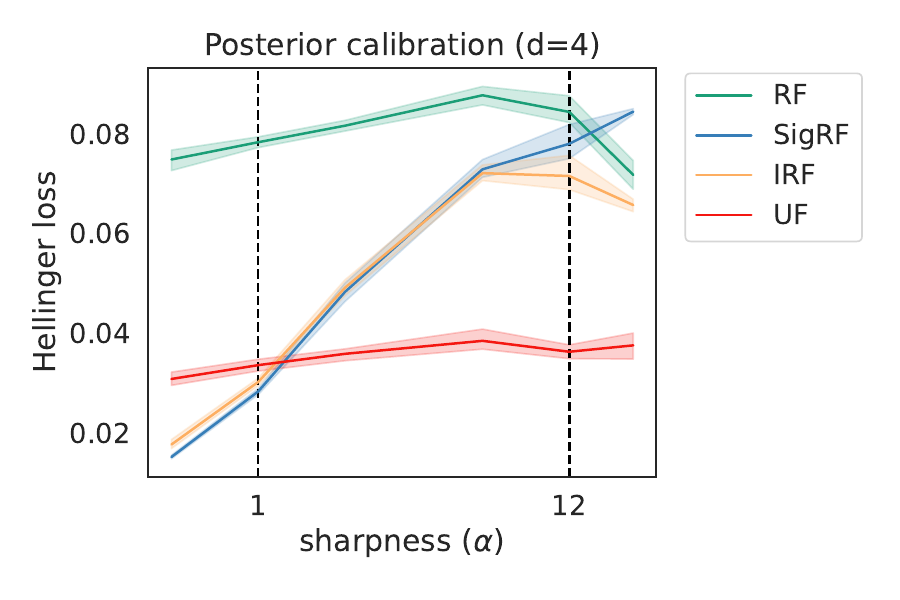}
    \includegraphics[width=0.56\linewidth]{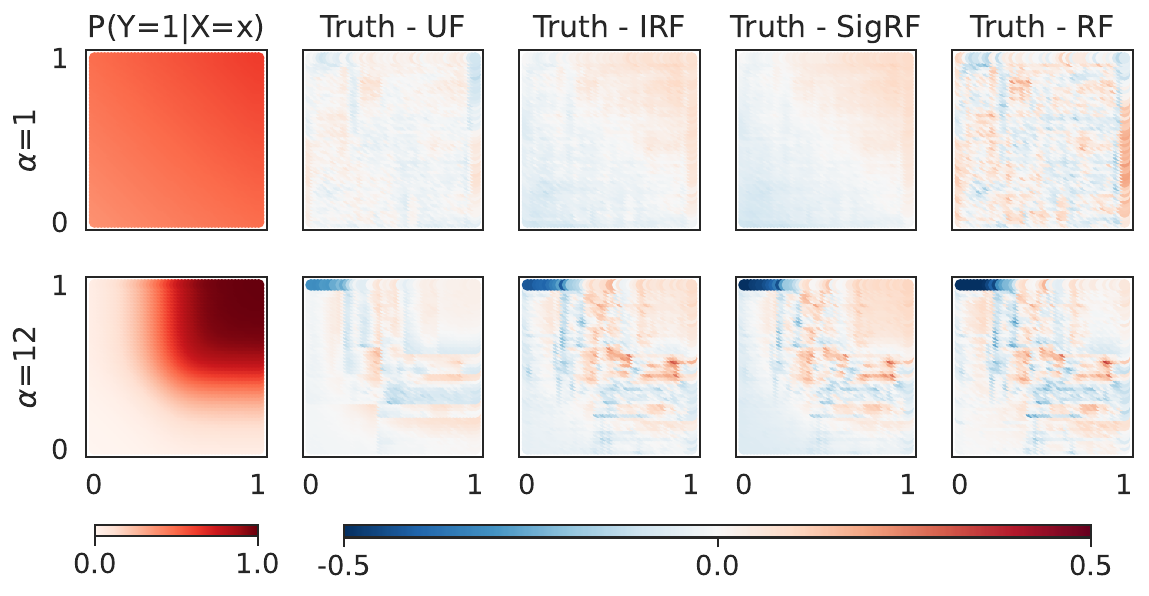}
    \caption{(\textbf{Left}) UF dominates RF, SigRF, and IRF in expected Hellinger distance to the true posterior in this simulated example across larger values of $\alpha$. (\textbf{Right}) At higher $\alpha$, the posteriors transition faster and UF predictions are closest to the truth, noticeably at the decision boundary. $5$,$000$ points in four dimensions were sampled for training, each of three times to compute the mean and variance of the estimate.}
    \label{fig:steep_posteriors}
\end{figure}

\subsubsection{OpenML CC18 Classification Tasks: Accuracy and Calibration}\label{sec:cc18}
We additionally evaluate RF, IRF, SigRF and UF on the 72 CC18 classification datasets from OpenML \citep{feurer_openml-python_2019, vanschoren_openml_2014} to complement our simulations. To our knowledge, while there have been many examples of honest forests for regression, our results here are the first examination of honest forests for posterior prediction and classification. The lengthy details and results are listed in full in Appendix \ref{sec:supp_cc18}, but the key results are shown in Table \ref{tab:cc18_tables}.

\begin{table}[htb!]
   \centering
  \begin{tabular}{| c | c | c | c | c || c | c | c | c |}
    \cline{2-9}
    \multicolumn{1}{c|}{} & 
    \multicolumn{4}{c||}{\textbf{Cohen's kappa Label Error}} &
    \multicolumn{4}{c|}{\textbf{Calibration Error (ECE)}}
    \\
    \cline{2-9}
     \multicolumn{1}{c|}{} & RF & IRF & SigRF & UF 
     & RF & IRF & SigRF & UF
     \\
    \hline
	 RF &  & -0.003* & -0.003* & -0.036* & & 0.017 & 0.004 & 0.014\\
    IRF & 0.003 & & 0.000* & -0.030* & -0.017* & & -0.016* & -0.007*\\
    SigRF & 0.003 & 0.000 & & -0.028* & -0.004 & 0.016 & & 0.002\\
    UF & 0.036 & 0.030 & 0.028 & & -0.014* & 0.007 & -0.002 & \\
    \hline
  \end{tabular}
  \caption{(\textbf{Left}) Median difference (row minus column) in Cohen's kappa misclassification error between each pair of methods across the CC18 datasets. (\textbf{Right}) Median pairwise difference in expected calibration error. A cell with a negative median difference indicates that the method in that row outperformed the method in that column. An asterisk means that the difference is significantly less than zero according to a one-sided pairwise Wilcoxon Sign test at the $\alpha = 0.05$ level. See Appendix \ref{sec:supp_cc18} for further details.}
  \label{tab:cc18_tables}
\end{table}
In terms of accuracy, RF performs the best, IRF and SigRF are approximately the same (which makes sense as they only differ in the post-training calibration) and UF performs the worst, as might be expected from its splitting of data. In terms of calibration error, as measured by the difference between the predicted posteriors and approximated true posteriors, IRF is the clear winner. Its adjustment to the posteriors is meant to directly minimize calibration error and is more flexible than SigRF. The next most calibrated method is UF, which shows a significant improvement over RF, unlike SigRF. Intuitively, it makes sense that UF is well calibrated since the honest posteriors in a single tree are independently and identically distributed as the test samples; thus, in a single tree the miscalibration is simply due to noisy and small leaves, although the effects of aggregation at the forest level are less obvious.

\section{Conditional Entropy and Mutual Information Estimation}\label{sec:mi_estimation}

Having detailed the capacity of decision forest methods to estimate posterior probabilities, we now introduce three approaches for computing conditional entropy and mutual information using these forests. We compare the forest-based estimates to the $k$-Nearest Neighbor based KSG methods in a variety of simulated examples and then derive MI estimates in the context of a \textit{Drosophila} connectome data set.

\subsection{Calculating Entropy Quantities via Decision Forests}
Given observations\\ $\mathcal{D}_n = \{(X_1, Y_1), ..., (X_n, Y_n)\}$, the goal is to estimate conditional entropy $H(Y|X)$
and thus mutual information by integrating posteriors learned from a decision forest over the feature space. To illustrate, imagine a trained forest $f(\cdot)$ providing posterior estimates and a set of samples $\{x'_1, \dots, x'_m\}$. Per the plug-in principle \cite[Section 2.1.2]{BandD_1977}, whereby we approximate a function on a distribution with the same function on the empirical distribution, we immediately obtain the estimate
$$H(Y|X) = \frac1m \sum_{i=1}^m H(Y|X = x'_i)] = -\frac1m \sum_{i=1}^m f(x'_i)\log f(x'_i).$$

In a conventional supervised setting, we are given just $\mathcal{D}_n$ and must both use the data to learn a classifier and to estimate the density. We clearly cannot use the same sample in $\mathcal{D}_n$ for both learning a decision tree and estimating $H(Y|X)$ since by learning the tree structure on that sample, we have overfit our model to it and our posteriors would be overconfident.

\subsubsection{Sample Splitting Estimates}
The first estimation strategy is the \textit{sample splitting} approach where we fit a forest to a subset $\{(X_i, Y_i)\}_{i=1}^m \subset \mathcal{D}_n$ and estimate $H(Y | X)$ using the remaining $\{(X_i, Y_i)\}_{i=m+1}^n$ samples, $m < n$. This solves our double-dipping problem but is an inefficient use of data, notably failing to use any of the labels $\{Y_{m+1}, \dots, Y_n\}$. The pseudocode is seen in Algorithm \ref{alg:mi_sample_splitting}.

\begin{algorithm}[htb!]
  \caption{Sample Splitting MI Estimation (see Appendix \ref{sec:code} for further details)}
  \label{alg:mi_sample_splitting}
\begin{algorithmic} 
\Require Training set $\mathcal{D}_n$, number of trees $B$, subsample size $s$, number of classes $K$
\State $\mathcal{D}^{\text{E}} \gets \textproc{RandomSubset}(\mathcal{D}_n, s)$ \Comment{Subsample estimation subset}
\State $\mathcal{D}^{\text{P}} \gets \mathcal{D}_n \setminus \mathcal{D}^{\text{E}}$ \Comment{Tree fitting subset}
\For{$b$ \textbf{in} B}
    \State $\mathcal{D}^{\text{P}}_{bootstrap} \gets \textproc{Bootstrap}(\mathcal{D}^{\text{P}}, n)$
    \State tree $\gets \textproc{FitDecisionTree}(\mathcal{D}^{\text{P}}_{bootstrap})$ \Comment{Structure and posteriors fit to this data}
    \For{$X_i$ \textbf{in} $\mathcal{D}^{\text{E}}$} \Comment{Posteriors on estimation subset}
        \State $\hat{p}_b(Y | X_i) \gets \textproc{EstimatePosterior}(tree, X_i)$
    \EndFor
\EndFor
\For{$X_i$ \textbf{in} $\mathcal{D}^{\text{E}}$} \Comment{Average tree-level posteriors}
    \State $\hat{p}(Y | X_i) \gets \frac1B \sum_{b=1}^B \hat{p}_b(Y | X_i)$
\EndFor
\State $\hat{H}(Y \mid X) \gets \frac{1}{|\mathcal{D}^{\text{E}}|} \sum_{X_i \in \mathcal{D}^{\text{E}}} \big[ - \sum_{y=1}^K \hat{p}(Y = y | X_i) \log \hat{p}(Y = y | X_i)\big]$
\State $\hat{H}(Y) \gets - \sum_{y=1}^K \hat{p}(y) \log \hat{p}(y), \quad \hat{p}(y) = \frac{1}{n} \sum_{i=1}^n \I[Y_i = y]$

\Ensure Mutual information $\hat{H}(Y) - \hat{H}(Y \mid X)$.
\end{algorithmic}
\end{algorithm}

\subsubsection{Out-of-Bag Estimates}
Alternatively, by using decision forests to estimate the posteriors we are able to cleverly leverage all of $\mathcal{D}_n$ without double-dipping. For each tree in the forest, $n$ samples are bootstrapped, which results in approximately $63.2\%$ unique samples from $\mathcal{D}_n$ being used in each tree on average \citep{efron_improvements_1997}. Thus, the `out-of-bag' (\textit{oob}) samples in each tree account for approximately $36.8\%$  of the samples from $\mathcal{D}_n$. Given a trained forest $f(\cdot)$, we can define the estimator $f_{oob}(x)$ as the forest composed of trees in $f$ for which $x$ was an \textit{oob} sample, so as not to double-dip into the training data. As \citet{Breiman2001} pointed out, the $f_{oob}(x)$ estimates of generalization error are unbiased estimates of the true generalization error. So, our \textit{oob} conditional entropy estimator is
$$H(Y|X) = \frac1n \sum_{i=1}^n H(Y|X = x_i)] = -\frac1n \sum_{i=1}^n f_{oob}(x)\log f_{oob}(x)$$
and leverages a forest trained on all of $\mathcal{D}_n$. This method applies to RF, SigRF, and IRF. It may also apply to UF if bagging is used prior to the honest sample splitting, which we omit so as to guarantee more samples for learning the structure and posteriors given the low correlation of trees due to honesty. Note that due to random sampling, while unlikely, there are no guarantees that a given training sample will ever be out-of-bag. The pseudocode is seen in Algorithm \ref{alg:mi_oob}.

\begin{algorithm}[htb!]
  \caption{\textit{oob} MI Estimation (see Appendix \ref{sec:code} for further details)}
  \label{alg:mi_oob}
\begin{algorithmic} 
\Require Training set $\mathcal{D}_n$, number of trees $B$, number of classes $K$
\For{$b$ \textbf{in} B}
    \State $\mathcal{D}^{\text{P}}_{bootstrap} \gets \textproc{Bootstrap}(\mathcal{D}_n, n)$ \Comment{Bootstrap from full data}
    \State $\mathcal{D}^{\text{P}}_{oob} \gets \mathcal{D}^{\text{P}} \setminus \mathcal{D}^{\text{P}}_{bootstrap}$ \Comment{Held-out \textit{oob} samples}
    \State tree $\gets \textproc{FitDecisionTree}(\mathcal{D}^{\text{P}}_{bootstrap})$ \Comment{Structure and posteriors fit to this data}
    \For{$X_i$ \textbf{in} $\mathcal{D}^{\text{P}}_{oob}$} \Comment{Posteriors on \textit{oob} samples}
        \State $\hat{p}_b(Y | X_i) \gets \textproc{EstimatePosterior}(tree, X_i)$
    \EndFor
\EndFor
\For{$X_i$ \textbf{in} $\mathcal{D}_n$} \Comment{Average tree-level \textit{oob} posteriors}
    \State $\hat{p}(Y | X_i) \gets \frac1B \sum_{b=1}^B \hat{p}_b(Y | X_i)$
\EndFor
\State $\hat{H}(Y \mid X) \gets \frac{1}{|\mathcal{D}_n|} \sum_{X_i \in \mathcal{D}_n} \big[ - \sum_{y=1}^K \hat{p}(Y = y | X_i) \log \hat{p}(Y = y | X_i)\big]$
\State $\hat{H}(Y) \gets - \sum_{y=1}^K \hat{p}(y) \log \hat{p}(y), \quad \hat{p}(y) = \frac{1}{n} \sum_{i=1}^n \I[Y_i = y]$

\Ensure Mutual information $\hat{H}(Y) - \hat{H}(Y \mid X)$.
\end{algorithmic}
\end{algorithm}

\subsubsection{Honest Estimates}
Due to the honesty property of UF, posteriors are fit using the subset of data $\mathcal{D}^V$, which is independent of the subset of data $\mathcal{D}^P$ used to learn a tree's structure. That is, each tree learns an informative, adaptive partition of the feature space, and the honest posteriors independently estimate the posterior probabilities in each partition. Thus, all of the data $\mathcal{D}_n$ can be used for density estimation without being biased due to overfit, homogeneous leaves, as exists in the other forest methods. The pseudocode is seen in Algorithm \ref{alg:mi_honest}.

\begin{algorithm}[htb!]
  \caption{Honest MI Estimation, no subsampling (see Appendix \ref{sec:code} for further details)}
  \label{alg:mi_honest}
\begin{algorithmic} 
\Require Training set $\mathcal{D}_n$, number of trees $B$, honest subsample size $s$, number of classes $K$
\For{$b$ \textbf{in} B}
    \State $\mathcal{D}^{\text{V}} \gets \textproc{RandomSubset}(\mathcal{D}_n, s)$ \Comment{Honest voters}
    \State $\mathcal{D}^{\text{P}} \gets \mathcal{D}_n \setminus \mathcal{D}^{\text{V}}$ \Comment{Tree fitting subset}

    \State tree $\gets \textproc{FitDecisionTree}(\mathcal{D}^{\text{P}})$
    \State tree $\gets$ $\textproc{FitPosteriors}(tree, \mathcal{D}^{\text{V}})$ \Comment{Set posteriors using honest samples}
    \For{$X_i$ \textbf{in} $\mathcal{D}_n$} \Comment{Posteriors on \textit{honest} samples}
        \State $\hat{p}_b(Y | X_i) \gets$ $\textproc{EstimatePosterior}(tree, X_i)$
    \EndFor
\EndFor
\For{$X_i$ \textbf{in} $\mathcal{D}_n$} \Comment{Average tree-level honest posteriors}
    \State $\hat{p}(Y | X_i) \gets \frac1B \sum_{b=1}^B \hat{p}_b(Y | X_i)$
\EndFor
\State $\hat{H}(Y \mid X) \gets \frac{1}{|\mathcal{D}_n|} \sum_{X_i \in \mathcal{D}_n} \big[ - \sum_{y=1}^K \hat{p}(Y = y | X_i) \log \hat{p}(Y = y | X_i)\big]$
\State $\hat{H}(Y) \gets - \sum_{y=1}^K \hat{p}(y) \log \hat{p}(y), \quad \hat{p}(y) = \frac{1}{n} \sum_{i=1}^n \I[Y_i = y]$

\Ensure Mutual information $\hat{H}(Y) - \hat{H}(Y \mid X)$.
\end{algorithmic}
\end{algorithm}

Note that for all of these estimation strategies, once the posteriors are fit, $\hat{H}(Y \mid X)$ sums over all samples as an approximation of the density. This calculation could just as easily sum over additional unlabeled samples, effectively generalizing to semi-supervised settings.

\subsection{Asymptotic Consistency of Honest Entropy Estimates}\label{sec:consistency_entropy}
Based on the consistency of UF's posterior estimates, proved in Section \ref{sec:consistency_posteriors}, we can now establish consistency of the estimate $\hat{H}_n(Y \mid X) = -\frac1n \sum_{i=1}^n \sum_{y \in \mathcal{Y}} p(y | x_i) \log p(y | x_i)$ of conditional entropy. The proofs to the following results are given in Appendix \ref{sec:proof}.

\begin{lem}[Consistency of the conditional entropy estimate]
    Given Assumptions 1-2, the conditional entropy estimate of UF built to Specification 1 is consistent as $n \rightarrow \infty$, that is, $\hat{H}_n(Y \mid X) \overset{P}{\rightarrow} H(Y \mid X)$.
    \label{thm:cond_ent}
\end{lem}
This result states that the estimate $\hat{H}_n(Y \mid X)$ is arbitrarily close in probability to the true $H(Y \mid X)$ for sufficiently large $n$. From the consistency of the maximum-likelihood estimate $\hat{H}_n(Y)$ and of $H(Y)$ based on empirical frequencies, we have the following key results.

\begin{theorem}[Consistency of the mutual information estimate] Given Assumptions 1-2, the mutual information estimate of UF built to Specification 1 is consistent as $n \rightarrow \infty$, that is, $\hat{I}_n(X;Y) \overset{P}{\rightarrow} I(X; Y)$ as $n \rightarrow \infty$.
\label{thm:mutual_info}
\end{theorem}

\subsection{Empirical Results: Conditional Entropy and Mutual Information Estimation}

We now show the capability of each of the forest methods to estimate conditonal entropy, compare them to the KSG estimators in a variety of simulated settings, and then examine a scenario with real data from the Drosophila larvae biological neural network (connectome).

In all cases, UF used honest estimation with $0.5n$ honest samples per tree (as in GRF \citep{Athey2019-uw}) while each of RF, SigRF, and IRF used both \textit{oob} estimation and sample splitting with $0.3n$ samples held out (approximately the proportion in the \textit{oob} set due to bagging). Each forest used all features in each tree due to signal sparsity and was composed of $300$ trees. Specifically, SigRF and IRF used $5$-fold internal cross validation and $60$ trees per fold

\subsubsection{Simulation: Conditional Entropy}\label{sim:conditional_entropy}

Figure \ref{fig:convergence} demonstrates how the forest methods perform and differ in a simple conditional entropy estimation task. Let $\mathcal{Y} = \{-1, 1\}$ for notational convenience and $P(Y = -1) = P(Y = 1) = 0.5$. $X$ is distributed according to the class-conditional isotropic multivariate Gaussian distribution $\mathcal{N}((Y\mu, 0, \dots, 0)^T, I_d)$, where $\mu$ is a parameter controlling effect size. Unlike the prior posterior estimation experiment where the models' predictions are evaluated on a large test sample, this simulation reflects a situation we would have in practice in which we simply have the data $\mathcal{D}_n$ and wish to estimate the conditional entropy. 

Because each dimension beyond the first is class-independent noise, the conditional entropy does not change. This allows us to compare our forest estimates to the true conditional entropy \citep{Ramdas15} as seen in Figure \ref{fig:convergence} where we examine each forest variant across varying sample size, effect size, and dimension. As the sample size increases, only $UF$ clearly appears to converge. All methods track the truth well as the effect size varies, except for SigRF which is biased throughout. As the dimension increases, all methods remain relatively close to their original estimate, but only UF appears the most accurate and invariant.

\begin{figure}[htb!]
  \centering
  \includegraphics[width=\linewidth]{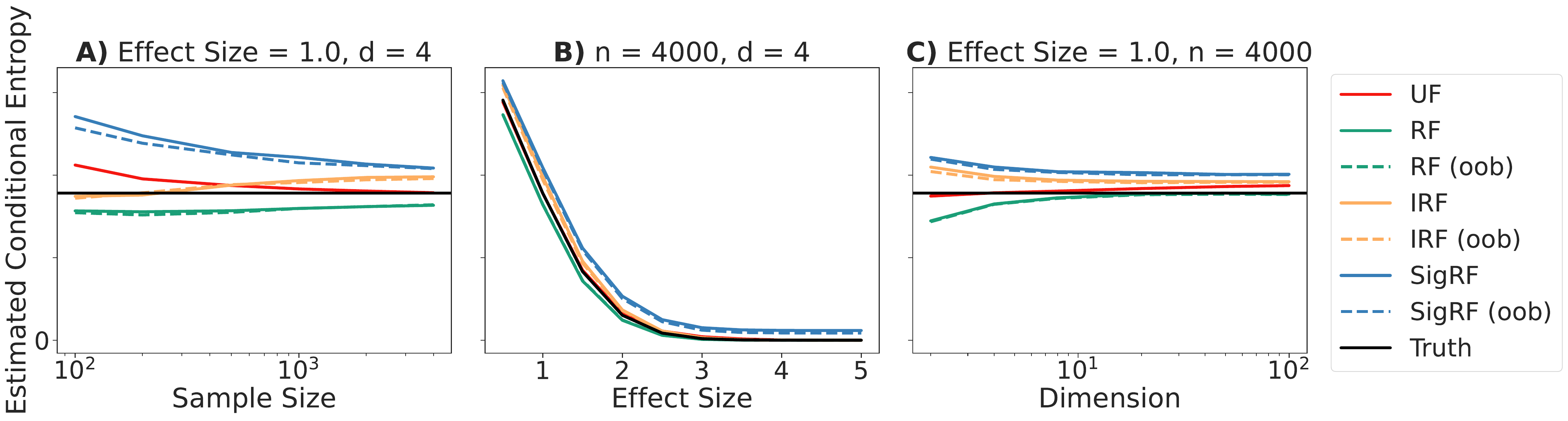}
  \caption{Decision forests estimate conditional entropy well in a variety of regimes. The average of $100$ runs is plotted. \textbf{A)} For $\mu=1$ and $d=4$, as $n$ increases, UF appears to converge while all other methods either don't or are slower. \textbf{B)} As the effect size increases, for fixed $n$ and $d$, all the methods appear to converge to the truth except for the SigRF variants which are biased. \textbf{C)} Even as the number of noise dimensions increases, all the methods closely track the truth, with UF appearing to be the most invariant.}
  \label{fig:convergence}
\end{figure}

\subsubsection{Simulation: Mutual Information Across Settings}\label{sec:mi_sim}
Now that we have established decision forests as adaptive nearest neighbor estimators of conditional information, and thus mutual information, we compare their performance in a variety of simulated settings to the popular $k$-NN estimators KSG \citep{ksg} and Mixed KSG \citep{mixedksg}. Figure \ref{fig:mi_sim} presents normalized mutual information $I(X; Y)/\min\{H(X), H(Y)\}$ estimates (calculated with the known true denominator for scaling purposes) in four different simulated settings as the class prior probabilities, feature dimensionality, and sample size vary. In each simulation setting, only the first one or two dimensions contain information about the class label while all others dimensions are standard normal Gaussian noise. Because the signal dimensions are Gaussian distributed, we can compute the true mutual information \citep{Ramdas15}. The specifics of each simulation in the first two dimensions are described below, all other dimensions being the aforementioned noise.

\paragraph{Overlapping Gaussians} A mixture of two identically distributed Gaussians such that \\$X \sim \mathcal{N}((0, 0)^T, I_2)$. Thus, the labels for each mixture component contain no information, and the true mutual information is 0.

As seen in the first row of Figure \ref{fig:mi_sim}, KSG appears to be approximately unbiased and yield a mutual information close to $0$, while Mixed KSG appears to perform worse in high dimensions and lower sample sizes. The other forest methods appear to experience greater bias but under a permutation test, in which the labels are permuted and a forest is retrained, the UF estimate is indistinguishable from the null distribution of independent labels. We conjecture that the constant bias seen in UF across $n$ is a result of increasingly small leaves (due to overfitting the noise) being sparsely populated by honest posteriors which compose a fixed fraction of $n$. UF lower bounds RF and \textit{oob} RF, however, which provide no calibration and so are even more confident and biased in their predictions.

\paragraph{Separated Gaussians} A mixture of two Gaussians from the same distribution as in Figure \ref{fig:convergence}. The effect size $\mu = 1$ and $\mathcal{Y} = \{-1, 1\}$ such that $X \mid Y=y \sim \mathcal{N}((y\mu, 0)^T, I_2)$.

As seen in the second row of Figure \ref{fig:mi_sim}, UF tracks the truth well over priors and at high dimensions, converging as $n$ gets large. All other forest methods display a noticeable bias and don't readily appear to converge, yet none degrade at high dimensions. In contrast, while the KSG estimators appear unbiased across class priors and converge in $n$, as the number of dimension gets large with added noise features their estimates degrade and converge to zero.

\paragraph{Three Class Gaussians} To simulate performance when there are more two classes, let $\mathcal{Y} = \{0,1,2\}$, and $X \mid Y = y \sim \mathcal{N}(\mb{\mu}_y, I)$
where $\mb{\mu}_0 = (0, \mu)\T, \enskip \mb{\mu}_1=  (\mu, 0)\T, I), \enskip \mb{\mu}_2 = (-\mu, 0)\T$.

As seen in the third row of Figure \ref{fig:mi_sim}, the trends are quite similar to the Separate Gaussian setting except that Mixed KSG seems to converge slightly faster than UF.

\paragraph{Scaled Gaussians} Another mixture of two Gaussians, but one where one Gaussian is scaled and thus non-isotropic.. Let $\mathcal{Y} = \{-1, 1\}$ and  $X \mid Y=y \sim \mathcal{N}\left((y\mu , 0)\T, \Sigma_y\right)$,
where $\Sigma_{-1} =
\begin{bmatrix} 
1/100 & 0 \\
0 & 1 
\end{bmatrix}$, and $\Sigma_{+1} = I_2$.

The Scaled Gaussian experiment highlights one of the key failures of the $k$-NN-based KSG and Mixed KSG estimators. When the data exists on different scales, such as the non-isotropic Gaussian presented here, the $L2$ distance becomes less informative. The forest methods, however, are invariant to monotone transformations of the data, which helps in this setting where the two classes are well separated; it would be even more evident if both Gaussians had equivalent covariances. As seen in the fourth row of Figure \ref{fig:mi_sim}, all forest methods appear robust to added noise dimensions and converge in sample size while UF and IRF additionally appear relatively unbiased. In contrast, the KSG methods show the greatest bias across all class priors, rapidly degrade in performance with added noise features, and are the slowest to converge in sample size.

\begin{figure}[htb!]
  \centering
  \includegraphics[width=\linewidth]{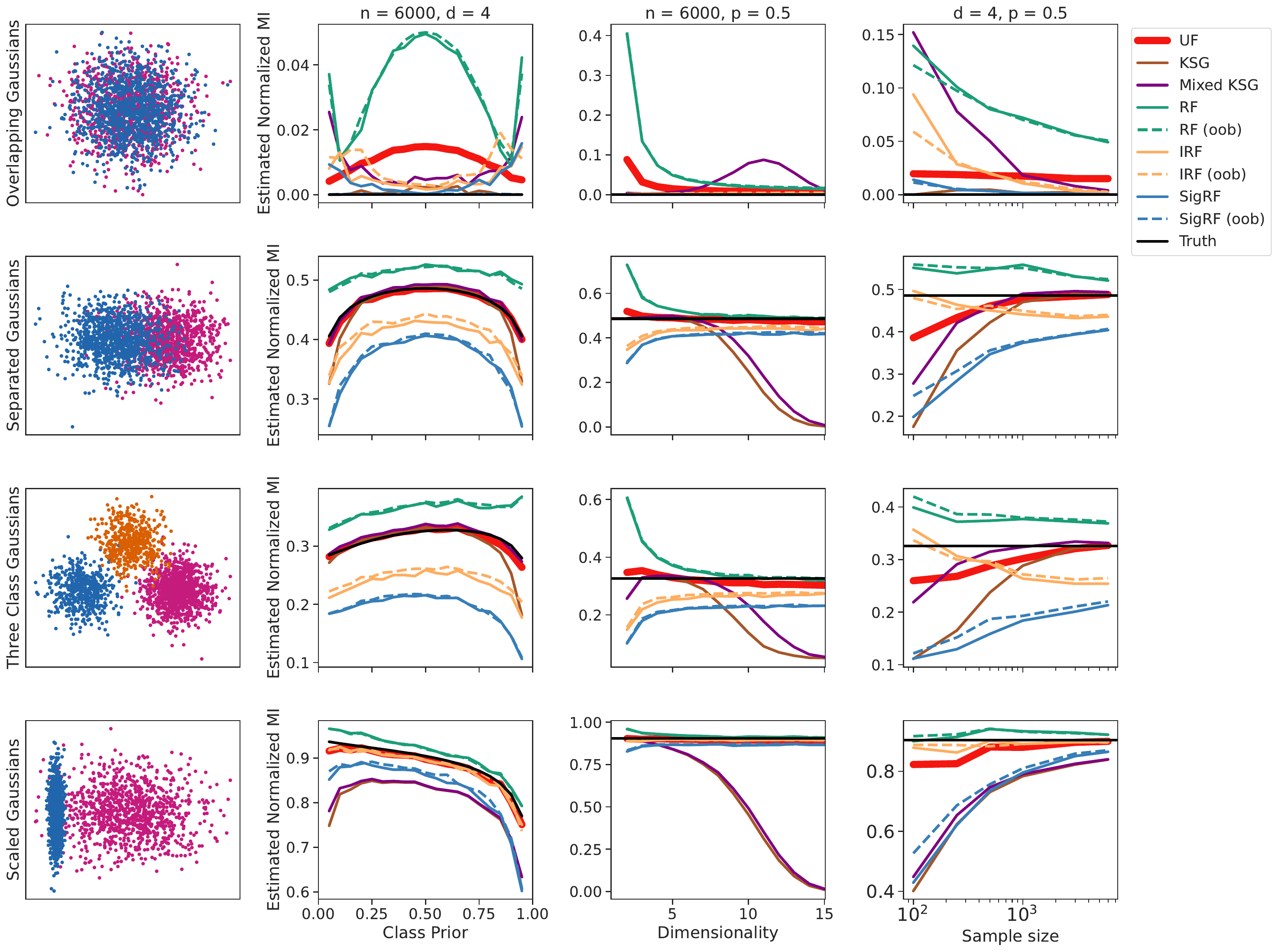}
  \caption{Mutual information estimates of forest methods and KSG methods in four simulated settings.
  (\textbf{Row 1}) The overlapping Gaussians produce zero MI. KSG appears to be zero in all cases while the forest methods display some bias and slow convergence, RF being the worst of all. However, a permutation test with UF fails to reject MI $> 0$.
  (\textbf{Row 2}) When the Gaussians are separated, the KSG methods and UF are all unbiased with UF converging the fastest. All other forest methods display prominent bias but are robust in high dimensions while the KSG methods degrade rapidly with added noise.
  (\textbf{Row 3}) With three separated Gaussian, the results seems to be the same as in the two Gaussian case.
  (\textbf{Row 4}) When the Gaussians are no longer all isotropic, the KSG methods converge much more slowly and degrade much faster with added noise dimensions. The forest methods do well across settings, converging in $n$ and being robust in $d$, with IRF and UF converging the fastest.
  }
  \label{fig:mi_sim}
\end{figure}

Overall, the forest methods are more robust than the KSG methods in higher dimensionality, noisy feature spaces, and when the data deviates from situations where the $L2$ distance is informative. Additionally, UF appears to be the most unbiased and fastest converging of the forest methods, which must be attributed to the honest posteriors. As discussed in Section \ref{sec:cc18}, UF leads to better calibrated posteriors on par with those of SigRF but seemingly worse than IRF. Yet here it is not readily apparent that successful calibration translates directly to successful entropy estimation. It appears that honest estimation provides real benefits while \textit{oob} estimation does not appear to greatly improve upon sample splitting, even at low sample sizes, although it may provide improvement in more complex feature spaces.

\subsubsection{Mutual Information in the Biological Drosophila Neural Network (Connectome)}\label{sec:mi_drosophila}

An immediate application of our random forest estimate of mutual information is measuring the information contained in biological covariates of neurons in the larval \textit{Drosophila} mushroom body (MB) \citep{mb} relative to their neuronal cell type. Part of the complexity of a brain comes from its constituent neurons existing not as independent entities but rather in a complex network structure, the connectome \citep{stat_connectomics}. Extracting meaningful statistics from the connectome, and understanding the relevance of those statistics is an ongoing challenge in computational neuroscience \citep{stat_connectomics}.

The dataset we examine here, obtained via serial section transmission electron microscopy, provides a real and important opportunity for investigating synapse-level structural connectome modeling \citep{Vogelstein2019-om}. This connectome consists of 213 neurons ($n = 213$) in four distinct cell types: Kenyon cells, input neurons, output neurons, and projection neurons. The connectome adjacency matrix is visualized in Figure \ref{fig:app} (left). Each neuron comes with a mixture of categorical and continuous features: ``claw'' refers to the integer number of dendritic claws for Kenyon cells, ``dist'' refers to real distance from the neuron to the neuropil, ``age" refers to neuron (normalized) age as a real number between -1 and 1, and ``cluster" refers to the community detected by the latent structure model as in \citet{lsm}, a function of the connectome adjacency matrix. We wish to determine if the latent structure model successfully captures topological information relevant to the neuronal cell type and, if it does, how that information relates to the other cell features.

\begin{figure}
  \centering
  \raisebox{-0.425\totalheight}{\includegraphics[width=0.3\linewidth]{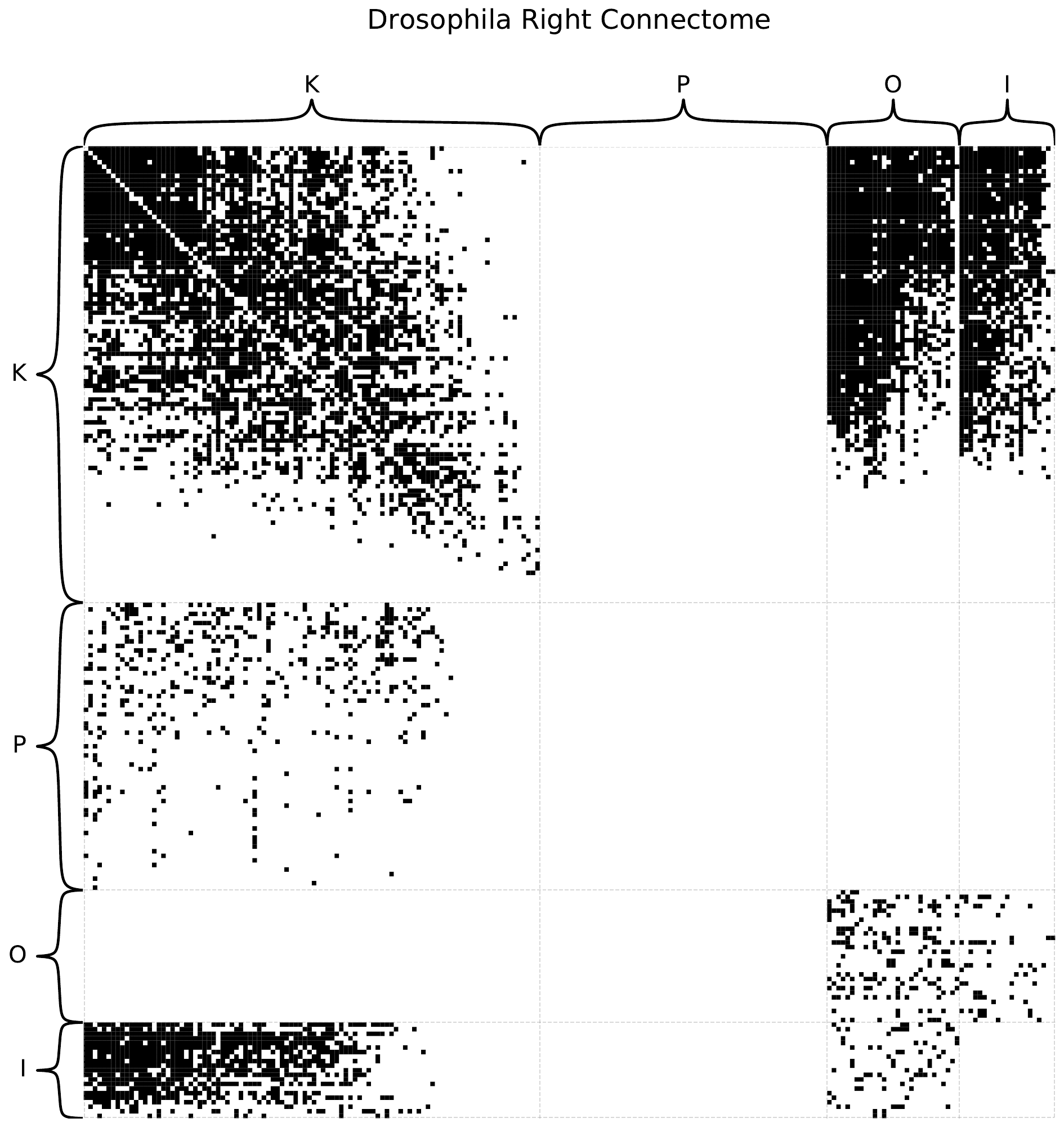}}
  \begin{tabular}{|c | c c c|}
    \hline
    $X_{\text{in}}$      & $\hat{I}(Y, X_{\text{in}})$           & $\hat{I}(Y, X_{\text{out}} \mid X_{\text{in}})$     & $\hat{I}(Y,X)$            \\
    \hline
    claw              & {\bf 0.362}      & 0.827  & 1.189 \\
    dist              & {\bf 0.542}      & 0.647  & 1.189 \\
    age               & {\bf 0.692}      & 0.497  & 1.189 \\
    cluster           & {\bf 1.102}      & 0.087  & 1.189 \\ \hline 
    cluster, claw      & 1.104   & 0.084  & 1.189 \\
    cluster, dist      & 1.121   & 0.068  & 1.189 \\
    cluster, age       & 1.189   & 0.000  & 1.189 \\
    cluster, claw, dist & 1.119  & 0.070  & 1.189 \\
    cluster, claw, age  & 1.189  & 0.000  & 1.189 \\
    cluster, dist, age  & 1.189  & 0.000  & 1.189 \\
    \hline
  \end{tabular}
  \caption{\textbf{Left} \textit{Drosophila} larva right hemisphere connectome. Groups of neurons are labelled $K$ for Kenyon Cells, $I$ for input neurons, $O$ for output neurons, and $P$ for projection neurons. Black cells represent the presence of an edge between the two corresponding nodes.
  \textbf{Right} Adjacency spectral embedding applied to the larval \textit{Drosophila} mushroom body (MB) connectome shows informative cluster groups for each neuron type, as seen by the large mutual information values. This suggests a strong dependency between neuron type and its structural position in the connectome.}
  \label{fig:app}
\end{figure}

We compute mutual information with $Y$ as the neuron type and $X$ as various subsets of the features. We apply UF as a robust estimator, as evident in the simulations. Because neuron type has been a subjective categorical assignment based on gross morphological features, we expect mutual information to be high for the entire feature vector. Indeed, UF estimated the mutual information to be $\hat{I}_n(Y, X) = 1.189$, statistically significantly greater than zero ($p$-value $< 0.001$)  under a permutation test of the categorical labels with 1000 re-trained forest replicates. Regarding the individual features, scientific prior knowledge posits a few relationships that are confirmed by the mutual information estimates. We compute $\hat{I}(Y, X_{\text{in}})$, the mutual information between $Y$ and the ``in" features $X_{\text{in}}$, and $\hat{I}(Y, X_{\text{out}} \mid X_{\text{in}})$, the additional information given by the ``out" features, per the chain rule of mutual information $\hat{I}(Y, X_{\text{in}}) + \hat{I}(Y, X_{\text{out}} \mid X_{\text{in}}) = \hat{I}(X,Y) $. The results are shown in Figure \ref{fig:app} (right) where we can see the information contained in the latent structure model ``cluster" feature and how it relates to other cell features.

Because only Kenyan cells have dendritic claws, this feature is unable to discriminate between the other three cell types, explaining why it has the lowest mutual information estimate among the features. 
Age is typically computed using distance from the neuropil, which explains why the ``age" and ``dist" features yielded relatively similar information regarding $Y$. Figure 4 of \citet{lsm} presents compelling evidence of latent structure model clusters being closely related to cell type, which is corroborated by its highest mutual information estimate among neuron features. Moreover, we observe that ``cluster" and ``age" together provide a unique minimal sufficient statistic encoding all the information in the data. That is, the ``claw" and ``dist" features provide no additional information on a neuron's cell type once the connectome and ``age" are known and provide insufficient information to replace either of these other features. Oftentimes it can be difficult to interpret features derived by complex methods from the high-dimensional connectome, yet UF here has demonstrated that the latent structure model results are indeed informative and subsume the information on neuronal cell type provided by the ``claw" and ``dist" features.

\section{Conclusion}
We presented a suite of decision forest methods as adaptive nearest neighbor estimators of conditional entropy and mutual information and as alternatives to the popular and established $k$-nearest neighbor KSG methods. The forest methods empirically performed well in high-dimensional settings and in cases where the $L2$ distance metric was not useful, both cases where the non-adaptive KSG methods were inadequate. Additionally, these methods maintained strong performance in low-dimensional settings and across sample sizes. Of the methods explored, Uncertainty Forest (UF), which uses the concept of \textit{honesty} to remove the bias from learned posteriors, provided the most reliable entropy estimates across settings and is the only forest method here with asymptotic convergence guarantees. Applied to a neuroanatomical dataset of the \textit{Drosophila} larva connectome, UF was able to provide informative, statistically significant mutual information estimates between various biological regions and \textit{a priori} labels, as well as validate features derived from the complex connectome structure.

As part of applying \textit{honesty} for conditional entropy and mutual information estimation, we presented a novel study of UF and other forest methods in the context of posterior probability estimation, evaluated on simulated and real datasets. In simple simulated settings, UF dominates when there is a sharp change in the posterior probabilities while the other forest methods overfit. Additionally, on par with IRF, UF provides better calibrated posteriors than RF and SigRF in regions of high uncertainty between classes. On the 72 CC18 classification tasks, we see that UF improves upon RF in terms of the calibration of posteriors at the cost of pure predictive power.

The main limitation of decision forests for conditional entropy and mutual information estimation is that the methods are only able to estimate uncertainty quantities for categorical $Y$; that said, the algorithm can be modified for continuous $Y$ as well.
Computing the posterior distribution when $Y$ is continuous can be accomplished with a kernel density estimate instead of simply binning the probabilities. A theoretical and empirical analysis of this extension is of interest. When $Y$ is multivariate, a heuristic approach such as subsampling $Y$ dimensions or using multivariate random forests could be explored. On the theoretic results, further important next steps also include rigorous proofs for convergence rates. On the applied results, feature selection in machine learning settings and dependence testing in high-dimensional, nonlinear scientific datasets will be natural applications for these decision forest methods. Fast permutation testing approaches will be necessary for efficient inference, especially in larger datasets.

\paragraph{Data and Code Availability Statement}
\label{sec:repo}

The implementation of UF, the simulated and real data experiments, and their visualization can be reproduced completely with the code available on Github (see https://neurodata.io/code/). 

\nocite{*}

\section*{Acknowledgements}

The authors are grateful for the support by the XDATA program of the Defense Advanced Research Projects Agency (DARPA) administered through Air Force Research Laboratory contract FA8750-12-2-0303, and DARPA's Lifelong Learning Machines program through contract FA8650-18-2-7834, and the National Science Foundation award DMS-1921310.

\vspace{5mm}
\bibliography{nd_template}
\bibliographystyle{unsrtnat}

\setcounter{theorem}{0}
\setcounter{lem}{0}
\setcounter{assumption}{0}
\clearpage
\appendix

 For additional pseudocode see Appendix \ref{sec:code}. For proofs and details of theory, see Appendix \ref{sec:proof}. For posterior prediction experiments, see Appendix \ref{sec:supplemental}.
\section{Pseudocode}
\label{sec:code}
The \textproc{FitDecisionTree},
and \textproc{ApplyTree} operations, as well as the \textproc{NumClasses} and \textproc{NumLeaves} fields are all standard functions in the \texttt{scikit-learn} decision tree and bagging classifier modules \citep{scikit-learn}. The following algorithm pseudocode complements the pseduocode in Section \ref{sec:mi_estimation}.

\begin{algorithm}
   \caption{Fitting Honest Posteriors}
   \label{alg:fit_posterior}
\begin{algorithmic} 
\Require Fitted tree, data to fit honest posteriors to $\mathcal{D}^V$
\Ensure tree with overwritten posterior probabilities
\Function{FitPosteriors}{tree, $\mathcal{D}^V$}
    \State $K$ = tree.\textproc{NumClasses}.
    \State $L$ = tree.\textproc{NumLeaves}.
    \State vote\_counts = $[0]^{L \times K}$.
    \For{observation $(x, y)$ \textbf{in} $\mathcal{D}^V$}
        \State $l$ = tree.\textproc{ApplyTree}($x$)
        \State vote\_counts[$l$, $y$] = vote\_counts[$l$, $y$] + 1.
    \EndFor
    
    \State posterior = $[0]^{L \times K}$.
    \For{leaf index $l$ \textbf{in} $[L]$}
        \State leaf\_size = $\sum_{y \in [K]}$ vote\_counts[$l$, $y$]
        \For{class $y$ \textbf{in} $[K]$}
            \State posterior[$l$, $y$] = vote\_counts[$l$, $y$] / leaf\_size
        \EndFor
    \EndFor
    tree.posterior = posterior
    \State \Return tree
\EndFunction
\end{algorithmic}
\end{algorithm}

\begin{algorithm}
   \caption{Predicting Posterior Probability}
   \label{alg:eval_posterior}
\begin{algorithmic} 
\Require fitted tree, evaluation sample $x$
\Ensure Posterior probabilities
\Function{EstimatePosterior}{tree, $x$}
    \State $K$ = tree.\textproc{NumClasses}.
    \State $l$ = tree.\textproc{ApplyTree}($x$).
    posterior = $[0]^K$
    \For{class $y$ \textbf{in} $[K]$}
        \State posterior[$y$] = tree.posterior[$l$, $y$].
    \EndFor
    \State \Return posterior
\EndFunction
\end{algorithmic}
\end{algorithm}

\clearpage
\section{Proofs}\label{sec:proof}

In this section, we present consistency results regarding estimation of conditional entropy and mutual information via Uncertainty Forest. The argument follows as a nearly direct consequence of \citet{Athey2019-uw}, who demonstrate that for some estimand $\theta(x)$ defined as the solution to a locally weighted estimating equations satisfying Assumptions 1A-6A (listed below), random forests built to Specification 1 provide plug-in estimates $\hat \theta(x)$ for $\theta(x)$ that are consistent and asymptotically Gaussian.

\subsection{Proof of Lemma \ref{thm:prob} (Consistency of the Posterior Probabilities)}

In order to apply the results of \citet{Athey2019-uw}, we need to first express UF's posterior probability estimand $\theta(x) = P(Y = k | X = x)$, for any $k \in \mathcal{Y}$ and all $x \in \mathcal{X}$, as a solution to a locally weighted estimating equation $M_\theta(Y) := \E[\psi_{\theta}(Y) | X = x] = 0$ where $\psi_\theta$ is some score function. To simplify notation, we suppress the dependency of $x$ and $k$ in $\theta(x)$ and simply write $\theta$ wherever this dependency can be deduced from the context. 

By equivalently reframing the estimand as the conditional mean $\theta = \E[\I[Y = k] | X = x]$, where $\I$ is the indicator function equal to $1$ when its argument is true and 0 otherwise, we can define
\begin{align*}
    M_\theta(Y) = \E[\I[Y = k] \mid X = x] - \theta(x) \quad\text{ and }\quad \psi_\theta(Y) = \I[Y = k] - \theta
\end{align*}

Having expressed the estimand appropriately, we must now address Assumptions 1A-6A \citep{Athey2019-uw}. We enumerate them below and address them inline.
\begin{enumerate}
    \item[1A.] For fixed values $\theta(x)$, we assume that $M_\theta(x)$ is Lipschitz continuous in $x$.
    This remains a true assumption on the distribution and is reframed as Assumption \ref{assmp:lipshitz}.

    \item[2A.] When $x$ is fixed, we assume that $M_\theta(x)$ is twice continuously differentiable  in $\theta$ with a uniformly bounded second derivative, and that $\frac{\partial}{\partial(\theta)} M_\theta(x) \mid_{\theta(x)} \neq 0$ for all $x \in \mathcal{X}$.
    This is clearly true when we evluate the derivatives $\frac{\partial^2 M_\theta}{\partial \theta^2} = 0$ and $\frac{\partial}{\partial \theta}M_\theta = -1$.

    \item[3A.] The worst-case variogram of $\psi_\theta (Y)$ is Lipschitz-continuous in $\theta(x)$.
    This is trivially true as seen by the expansion
    \begin{align*}
        \sup_{x \in \mathcal{X}}(\text{Var}[\psi_\theta(Y) - \psi_\theta(Y)]) &= \sup_{x \in \mathcal{X}}(\text{Var}[\I[Y = y] - \theta - (\I[Y = y] - \theta')])
        &= \sup_{x \in \mathcal{X}}(\text{Var}[\theta' - \theta'])
        &= 0,
    \end{align*}

    \item[4A.] The $\psi$-functions can be written as $\psi_{\theta}(Y) = \lambda(\theta(x); Y) + \xi_{\theta(x)}(g(Y))$, such that $\lambda$ is Lipschitz-continuous in $\theta$, $g: \{Y\} \rightarrow \Real$ is a univariate summary of $Y$, and $\xi_{\theta(x)} : \Real \rightarrow \Real$ is any family of monotone and bounded functions.
    The function $\psi_\theta(Y)$, a function of $Y$  and $\theta$, is itself Lipschitz in $\theta$.

    \item[5A.] For any weights $\alpha_i(x)$ such that $\sum_i \alpha_i(x) = 1$, the estimation equation returned a minimizer $\hat{\theta(x)}$ that at least approximately solves the estimating equation $||\sum_{i=1}^n \alpha_i(x) \psi_{\hat\theta} (Y_i)||_2 \leq C max\{\alpha_i(x)\}$ for some constant $C \geq 0$.
    The solution of
    \begin{align*}
        \sum_{i=1} \alpha_i(x) \psi_{\theta}(Y_i) = 0
    \end{align*}
    in $\theta$ exists, and is equal to
    \begin{align*}
        \hat{\theta}(x) = \sum_{i=1}^n \alpha_i(x) \I[Y_i = k]
    \end{align*}

    \item[6A.] The score function $\psi_{\theta} (Y)$ is a negative sub gradient of a convex function, and the expected score $M_\theta(x)$ is the negative gradient of a strongly convex function.
    Choose
    \begin{align*}
        \bf{\Psi}(\theta) &= \frac{1}{2}(\I[Y = y] - \theta)^2,\\
        \bf{M}(\theta) &= \frac{1}{2}(p(y \mid x) - \theta)^2
    \end{align*}
    as these functions.
\end{enumerate}

The $\alpha_i(x)$'s  weigh highly observations with $X_i$ close to $x$, as to approximate the conditional expectation of $\psi_\theta(Y)$, or $M_\theta(Y)$. In a random forest, these weights are defined to be the empirical probability that the test point $x$ shares a leaf with each training point $X_i$. As pointed out in \cite{Athey2019-uw}, in the case of conditional mean estimation, $\hat \theta$ is equivalent to canonical average of estimates over trees, the estimator of UF. As we have satisfied or kept Assumptions 1A-6A by construction, $\hat \theta$ is a consistent estimate of $\theta$ \citep{Athey2019-uw} and so is UF.

\subsection{Proof of Lemma \ref{thm:cond_ent} (Consistency of the conditional entropy estimate)}

By Lemma \ref{thm:prob} and continuity, we prove consistency of the honest forest estimate of conditional entropy. First, we prove an intermediate results. Let $\hat{H}_n(Y \mid X = x) = -\sum_{y \in [K]} \hat{p}_n(y \mid x) \log \hat{p}_n(y \mid x)$.

\begin{corollary} For each $x\in \Real^d$, we have that $\hat{H}_n(Y \mid X = x) \overset{P}{\rightarrow} H(Y \mid X = x).$
\label{thm:no_finite}
\end{corollary}
\begin{proof}
    The function 
    \begin{align*}
        h(p) = \begin{cases}
        0 &\text{ if } p = 0\\
        -p \log p &\text{ otherwise}
        \end{cases}
    \end{align*}
    is continuous on $[0,1]$. Similarly, the finite sum $\sum_{k=1}^K h(p_k)$ is continuous on $\{(p_1, ..., p_K) : 0 \leq p_k \leq 1, \sum_{k=1}^K p_k = 1\}$. Thus, by Lemma \ref{thm:prob} and the continuous mapping theorem, we have the desired result.
\end{proof}

Now to prove Lemma \ref{thm:cond_ent}, we wish to show convergence in probability, and thus that $\forall \varepsilon, \delta > 0$, $\exists N > 0$ such that
$P(|\hat H_n(Y \mid X) - H(Y \mid X)| > \varepsilon) < \delta$
for all $n \geq N$, equivalently denoted by the limit taken in probability
$$\underset{n\to\infty}{\text{plim}} \bigg|\hat H_n(Y \mid X) - H(Y \mid X)\bigg| = 0.$$
The conditional entropy estimate is defined as $\hat{H}_{n}(Y \mid X) = \frac{1}{n} \sum_{i \in \mathcal{D}_n} \hat{H}_n(Y \mid X = X_i)$, and by Corollary \ref{thm:no_finite} we have the pointwise consistency of $\hat{H}_n(Y \mid X = x)$ for $H(Y \mid X = x)$.

Because of the finiteness of $[K]$, the entropy-like term $|\hat{H}_n(Y \mid X = x)|$ is bounded by $\log K$ for all $n$ and $x$. Thus the mean
$$\E_{X'}[\hat{H}_n(Y \mid X = X')] = \int_{x \in \Real^d} \hat{H}_n(Y \mid X = x) \ dF_X$$
of the finite sample estimator exists.
By the triangle inequality, we have that
\begin{align*}
    \bigg|\frac1n \sum_{i \in \mathcal{D}_n} \hat{H}_n(Y \mid X = X_i) - H(Y \mid X)\bigg| 
    &\leq \bigg|\frac1n \sum_{i \in \mathcal{D}_n} \hat{H}_n(Y \mid X = X_i) - \E_{X'}[\hat{H}_n(Y \mid X = X')]\bigg|\\
    &\quad+\bigg|\int_{x \in \Real^d} \hat{H}_n(Y \mid X = x) \ dF_X - H(Y \mid X)\bigg|
\end{align*}

With regards to the first term on the right, by the Weak Law of Large Numbers (WLLN)
\begin{align*}
    \underset{n\to\infty}{\text{plim}} \bigg|\frac1n \sum_{i \in \mathcal{D}_n} \hat{H}_n(Y \mid X = X_i) - \E_{X'}[\hat{H}_n(Y \mid X = X')]\bigg| = 0.
\end{align*}

With regards to the second term on the right side, since the conditional entropy is bounded and converges pointwise by Corollary \ref{thm:no_finite}, by the Dominated Convergence Theorem
\begin{align*}
    &\underset{n\to\infty}{\text{plim}} \bigg|\int_{x \in \Real^d} \hat{H}_n(Y \mid X = x) \ dF_X - H(Y \mid X)\bigg| \\
    =& \underset{n\to\infty}{\text{plim}} \bigg|\int_{x \in \Real^d} \hat{H}_n(Y \mid X = x) \ dF_X - \int_{x \in \Real^d} H(Y \mid X = x) \ dF_X \bigg| \\
    =& \underset{n\to\infty}{\text{plim}} \bigg|\int_{x \in \Real^d} \hat{H}_n(Y \mid X = x) - H(Y \mid X = x) \ dF_X \bigg|=0.
\end{align*}

Thus our result
\begin{align*}
    \underset{n\to\infty}{\text{plim}} \bigg|\frac1n \sum_{i \in \mathcal{D}_n} \hat{H}_n(Y \mid X = X_i) - H(Y \mid X)\bigg| = 0
\end{align*}
follows as the lower bound of the sum of the prior two vanishing quantities.

\subsection{Proof of Theorem \ref{thm:mutual_info} (Consistency of the mutual information estimate)}

For $i.i.d.$ observations of $Y_i$, it is clear that $\hat{H}_n(Y)$ is a consistent estimate for $H(Y)$. By Lemma \ref{thm:cond_ent} we have the consistency of $\hat{H}_n(Y \mid X)$ for $H(Y \mid X)$. Thus, the consistency of $\hat{I}_n(X,Y)$ follows immediately as the sum of two consistent terms.
\clearpage
\section{Supplemental Results}
\label{sec:supplemental}

To our knowledge, honest sampling for random forest classification and posterior estimation has not been explored before, despite the recent prevalence of honest regression forests. To supplement our theory and use of honest posteriors for entropy estimation, here we provide a more detailed exploration of honesty in the context of classification.

\subsection{Simulation: Steep Posteriors}
\label{sec:supp_steep}

The details of this simulation are fully detailed in Section \ref{sim:steep_posteriors}. Additional plots here are shown in Figure \ref{fig:supp_steep_hellinger} for Hellinger distance and Figure \ref{fig:supp_steep_heatmap} for estimated posteriors across additional values of $\alpha$. Each of the decision forest methods were composed of $500$ trees, $5$ folds of $100$ trees for IRF and SigRF, and used all of the features in each split due to the sparsity of the signal. Here, UF used $0.63$ of the samples to learn the tree structure to better equate to the number of unique samples used during bootstrapping, although our experiences indicated that the larger value was not needed in this relatively simple feature space.

\begin{figure}[!htb]
  \centering
  \includegraphics[width=0.49\linewidth]{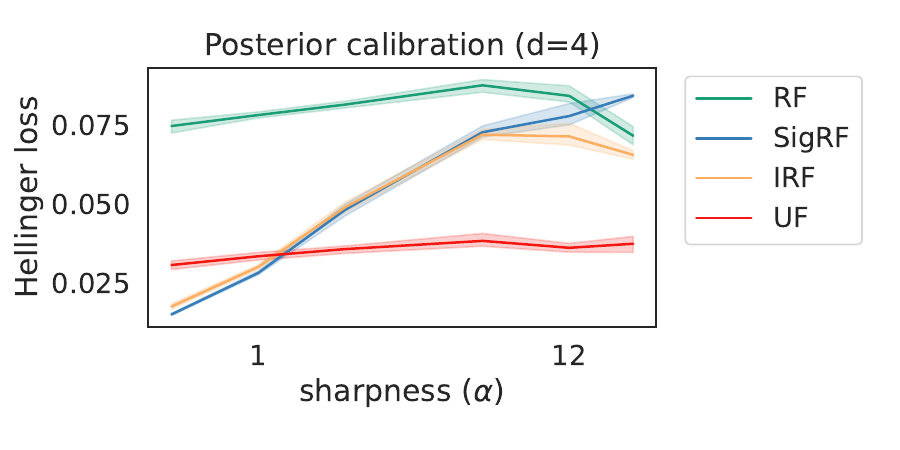}
  \includegraphics[width=0.49\linewidth]{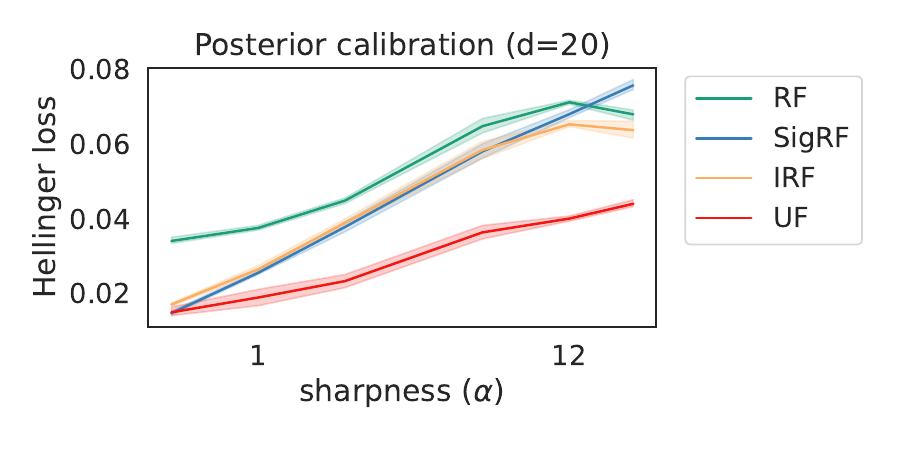}
  \caption{UF dominates RF, SigRF, and IRF in Hellinger distance to the true posterior of this simulated example across most transition sharpness parameter values of $\alpha$ in dimension $d=4$ (Left) and across all values of $\alpha$ in dimensions $d=20$ (Right).}
  \label{fig:supp_steep_hellinger}
\end{figure}

\begin{figure}
  \centering
  \includegraphics[width=\linewidth]{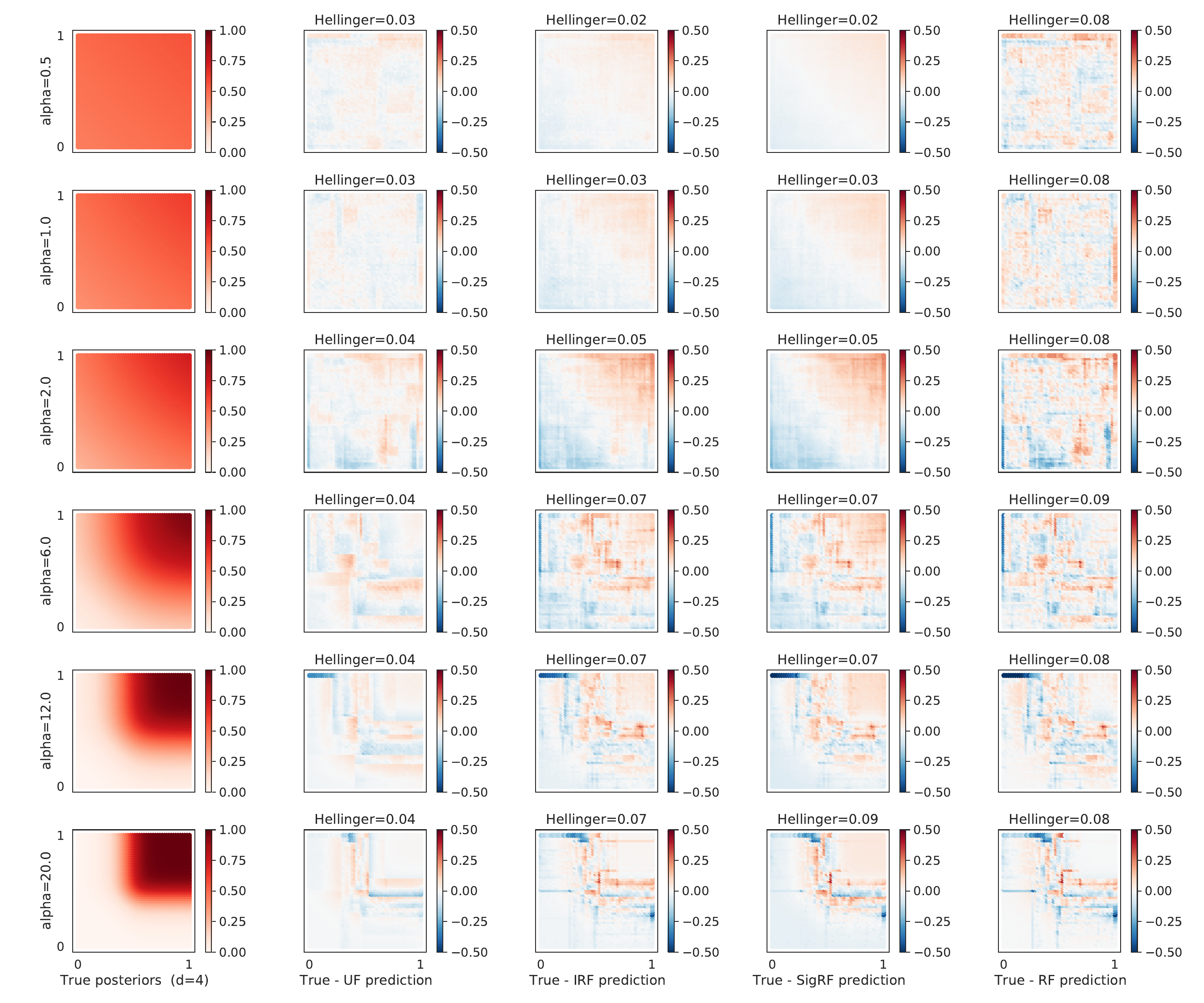}
  \caption{As $\alpha$ increases (down the rows), the simulated posteriors in the leftmost column become more sharp. We see in each of the subsequent columns the difference between the true posterior and the forest posterior over a grid in the first two dimensions. UF appears to do the best at reducing error near the posterior transition and near the boundary of the feature space.}
  \label{fig:supp_steep_heatmap}
\end{figure}

\subsection{OpenML CC18 Classification Tasks: Accuracy and Calibration}
\label{sec:supp_cc18}
We additionally evaluated the decision forest posteriors on the 72 CC18 classification datasets from OpenML \citep{feurer_openml-python_2019, vanschoren_openml_2014} to complement our simulations. Each of the forests were trained as in the prior Section \ref{sec:supp_steep} except since the amount of signal vs noise is unknown, each forest searched over $0.33$ of the features per split node. Datasets with missing values were imputed using either the median (in the case of continuous data) or the mode (in the case of categorical data) and all categorical variables were subsequently one-hot-encoded. Each of the datasets were fit and evaluated using 5-fold cross validation and the average scores are shown in our results. Because we care both about prediction accuracy and calibration, we evaluate each forest using the following metrics.
A test dataset with $m$ observations in $K$ classes has true labels $y = \{y_i\}_{i=1}^m$ and a decision forest predicts posteriors $p_i(y = k)$ for $i \in \{1,\cdots, m\}$ and $k \in [K] := \{1, \cdots, K\}$. The predicted class $\hat{y} = \{\argmax_{k \in [K]} p_i(y = k)\}_{i=1}^m$ has the largest posterior probability, denoted $\hat{p}_i$.

\begin{description}
\item[Cohen's kappa loss \citep{mchugh_interrater_2012}:] Cohen's kappa measures the agreement between the true and predicted labels but accounts for the accuracy of random guessing unlike $01$-loss. We scale the typical Cohen's kappa score by $-1$ so that it is minimized by a perfect algorithm. The resulting loss is defined as
$-\frac{p_0 - p_e}{1 - p_e}$
where $p_0$ is the $01$-accuracy and $p_e$ is the chance accuracy. See Figure \ref{fig:ck_wilcoxon}  for the pairwise forest comparison across CC18 datasets using the Wilcoxon Sign test and Figure \ref{fig:supp_ck_stripplot} for the raw Cohen\'s kappa losses per dataset.

\item[Expected Calibration Error (ECE) \citep{guo_calibration_2017}:] ECE scores how well the posterior of the predicted class equals an estimate of the true posterior, since it is generally unknown. The test samples are divided into $L$ bins of width $1/L$ on $(0, 1]$ where bin $B_l = \{i: \hat{p}_i \in (\frac{1-l}{L}, \frac{l}{L}]\}, l \in [L]$. The bin accuracy is $acc(B_l) = \frac{1}{|B_l|} \sum_{i \in B_l} \II[\hat{y}_i = y_i]$. The bin confidence is $conf(B_l) = \frac{1}{|B_l|} \sum_{i \in B_l} \hat{p}_i$ and equals the corresponding bin accuracy in a perfectly calibrated model. So $ECE = \sum_{l \in [L]} \frac{|B_l|}{m} |acc(B_l) - conf(B_l)|$ measures the weighted average difference. See Figure \ref{fig:ece_wilcoxon} for the pairwise forest comparison across CC18 datasets using the Wilcoxon Sign test and Figure \ref{fig:supp_ece_stripplot} for the raw ECEs per dataset using 20 bins.

\item[Maximum Calibration Error (MCE) \citep{guo_calibration_2017}:] MCE scores the worst disagreement between the bin confidence and bin accuracy defined as for ECE. Formally, $MCE = \max_{l \in [L]} |acc(B_l) - conf(B_l)|$. An important caveat is that MCE does not account for the sizes of the bins and so may be sensitive to noisy, small bins. See Figure \ref{fig:mce_wilcoxon}  for the pairwise forest comparison across CC18 datasets using the Wilcoxon Sign test and Figure \ref{fig:supp_mce_stripplot} for the raw MCEs per dataset using 20 bins.

\end{description}

Note that Hellinger distance is not a possible metric here as the true posteriors are not known. Additionally, the ECE and MCE scores and plots evaluate only the posterior of the predicted class to be applicable to tasks with any number of classes. This is in contrast to the two-class version of ECE used in \citet{niculescu-mizil_predicting_2005}. ECE and MCE are also approximations that use simple histogram binning to estimate densities, in comparison to regression based smoothing estimators of these quantities \citep{austin_integrated_2019}. In practice, while regression based smoothers may be more accurate, for large test samples the ECE and MCE approximations are not noticeably different but computed much more quickly.

ECE can be better visually explored by plotting the bin accuracy against the bin confidence. An ECE of 0 means that each point would lie on the line $y=x$. The fraction of samples in each of the bins approximates the density of confidences. These visualization, for each of of the CC18 datasets, are plotted in Figures \ref{fig:supp_ece_density_1} and  \ref{fig:supp_ece_density_2}.

The key takeaways seem to be as follows. In terms of accuracy, RF performs the best, IRF and SigRF are approximately the same, which makes sense as they only differ in the post-training calibration, and UF performs the worst, as would be expected from its splitting of data. In terms of the calibration metrics ECE and MCE, IRF is the clear winner. Its adjustment to the posteriors is meant to minimize these metrics and is more flexible than SigRF. The next most calibrated method actually appears to be UF, which is significantly different than RF, unlike SigRF. Intuitively, it makes sense that UF is well calibrated as the honest posteriors in a single tree are independently and identically distributed as the test samples; thus in a single tree the miscalibration is simply due to noisy and small leaves, although the effects of aggregation at the forest level are less obvious.

\clearpage

\begin{table}[!ht]
  \centering
  \begin{tabular}{|c | c | c | c | c|}
    \hline
    & RF & IRF & SigRF & UF \\
    \hline
	RF &  & -0.003 (0.000) & -0.003 (0.000) & -0.036 (0.000) \\
    IRF & 0.003 (1.000) & & 0.000 (0.038) & -0.030 (0.000) \\
    SigRF & 0.003 (1.000) & 0.000 (0.962) & & -0.028 (0.000) \\
    UF & 0.036 (1.000) & 0.030 (1.000) & 0.028 (1.000) & \\
    \hline
  \end{tabular}
  \caption{Median row-column difference in Cohen's kappa error across all datasets, with one-sided Wilcoxon Sign test p-value in parentheses. A cell with a negative value indicates that the method in that row outperformed the method in that column.}
  \label{fig:ck_wilcoxon}
\end{table}

\begin{figure}[!hb]
  \centering
  \includegraphics[width=0.64\linewidth]{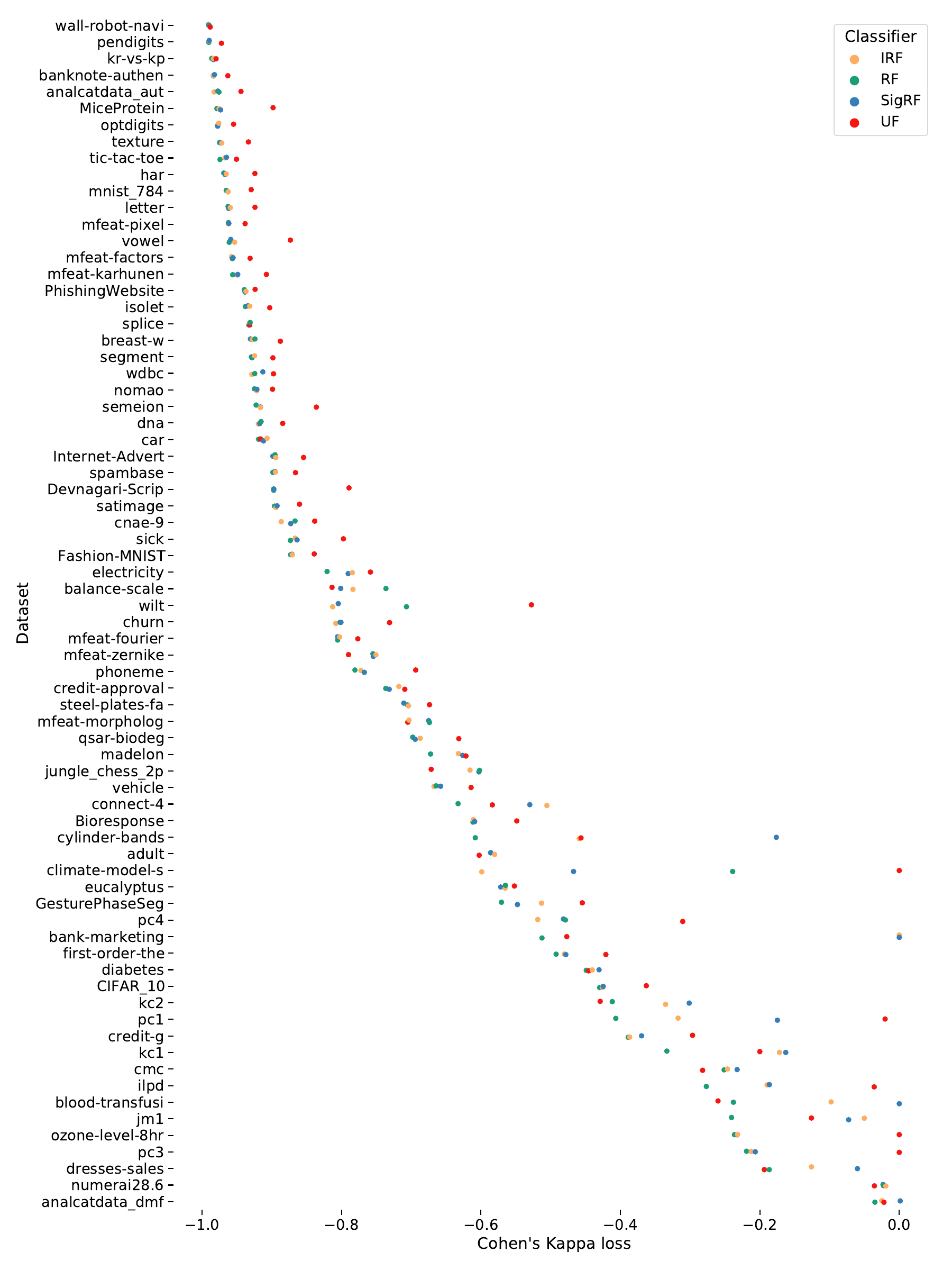}
  \caption{The Cohen's kappa loss for each of the decision forests in each of the 72 CC18 datasets, averaged over 5 cross-validation folds. A value of $-1$ means perfect accuracy whereas a value greater than or equal to $0$ means chance accuracy.}
  \label{fig:supp_ck_stripplot}
\end{figure}

\newpage
\begin{table}[!ht]
  \centering
  \begin{tabular}{|c | c | c | c | c|}
    \hline
    & RF & IRF & SigRF & UF \\
    \hline
	RF & & 0.017 (1.000) & 0.004 (0.936) & 0.014 (1.000) \\
    IRF & -0.017 (0.000) & & -0.016 (0.000) & -0.007 (0.003) \\
    SigRF & -0.004 (0.064) & 0.016 (1.000) & & 0.002 (0.938) \\
    UF & -0.014 (0.000) & 0.007 (0.997) & -0.002 (0.062) & \\
    \hline
  \end{tabular}
  \caption{Median row-column difference in ECE across all datasets, with one-sided Wilcoxon Sign test p-value in parentheses. A cell with a negative value indicates that the method in that row outperformed the method in that column.}
  \label{fig:ece_wilcoxon}
\end{table}

\begin{figure}[!hb]
  \centering
  \includegraphics[width=0.66\linewidth]{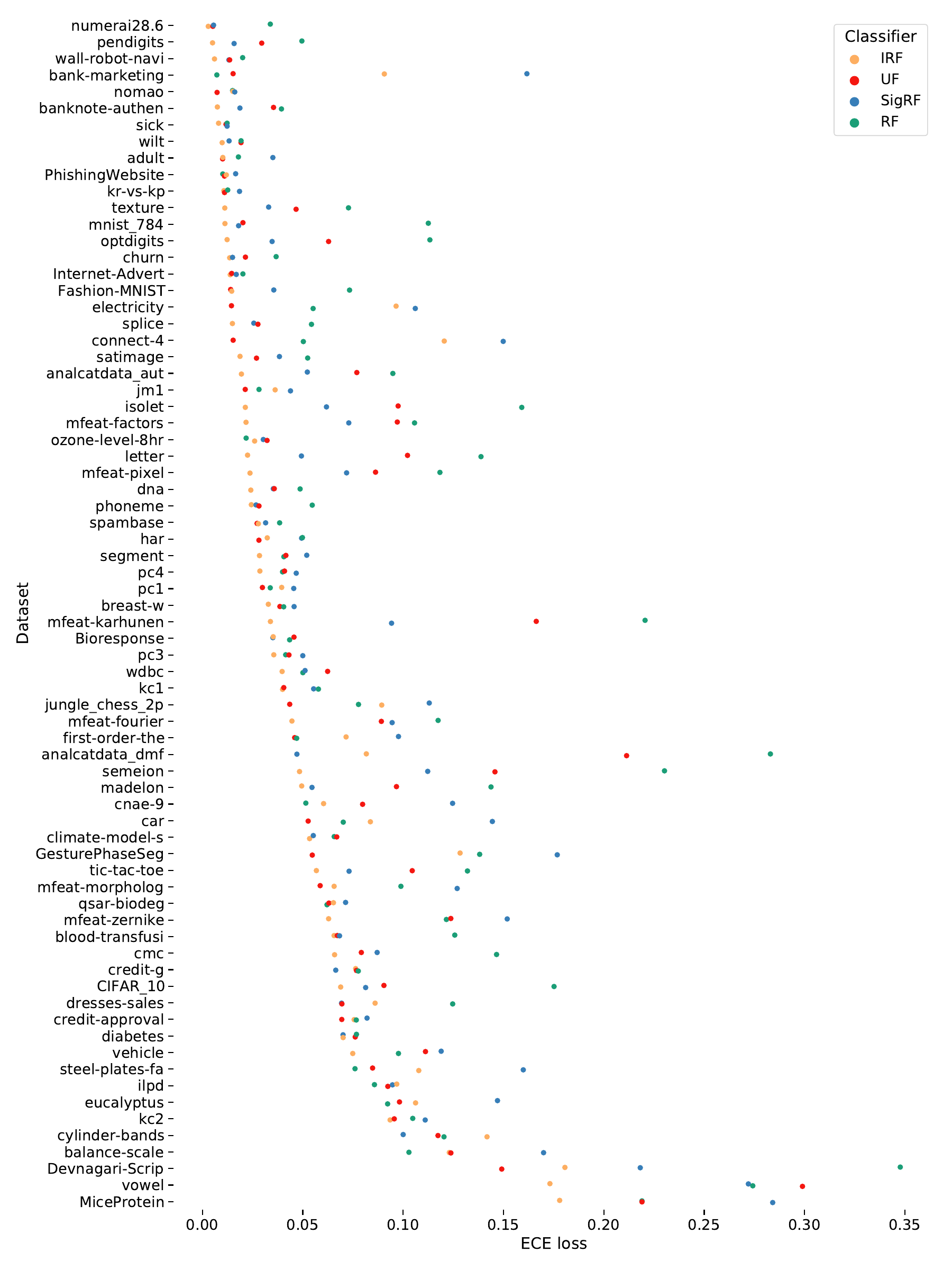}
  \caption{The ECE for each of the decision forests in each of the 72 CC18 datasets, averaged over 5 cross-validation folds. A value of $0$ means perfect calibration.}
  \label{fig:supp_ece_stripplot}
\end{figure}

\newpage
\begin{table}[!ht]
  \centering
  \begin{tabular}{|c | c | c | c | c|}
    \hline
    & RF & IRF & SigRF & UF \\
    \hline
	RF & & 0.029 (0.997) & 0.007 (0.827) & 0.037 (0.974) \\
    IRF	& -0.029 (0.003) & & -0.010 (0.027) & -0.006 (0.111) \\
    SigRF & -0.007 (0.173) & 0.010 (0.973) & & 0.003 (0.742) \\
    UF & -0.037 (0.026) & 0.006 (0.889) & -0.003 (0.258) & \\
    \hline
  \end{tabular}
  \caption{Median row-column difference in MCE across all datasets, with one-sided Wilcoxon Sign test p-value in parentheses. A cell with a negative value indicates that the method in that row outperformed the method in that column.}
  \label{fig:mce_wilcoxon}
\end{table}

\begin{figure}[!hb]
  \centering
  \includegraphics[width=0.66\linewidth]{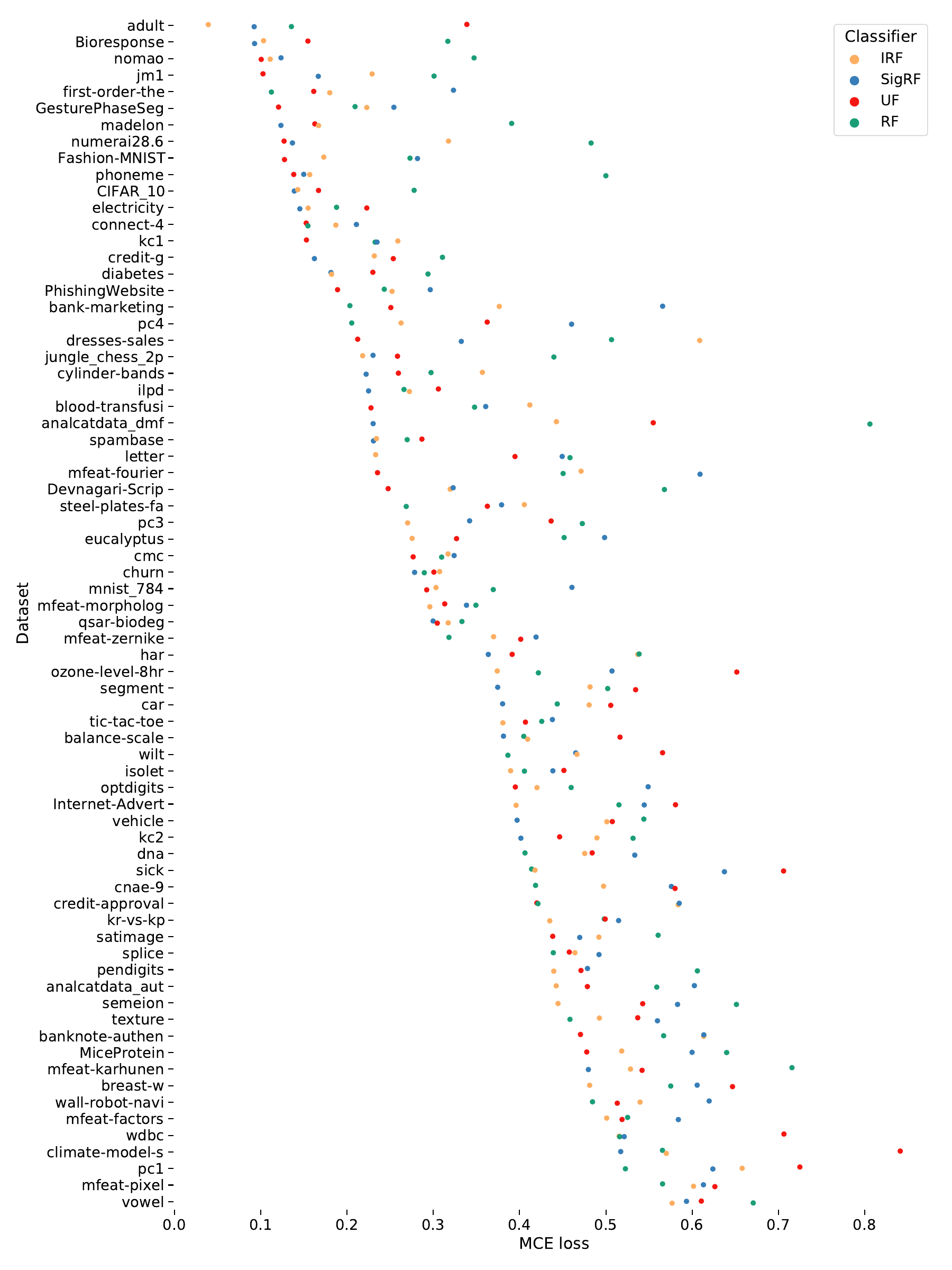}
  \caption{The MCE for each of the decision forests in each of the 72 CC18 datasets, averaged over 5 cross-validation folds. A value of $0$ means perfect calibration.}
  \label{fig:supp_mce_stripplot}
\end{figure}

\begin{figure}[!ht]
  \centering
  \includegraphics[width=\linewidth]{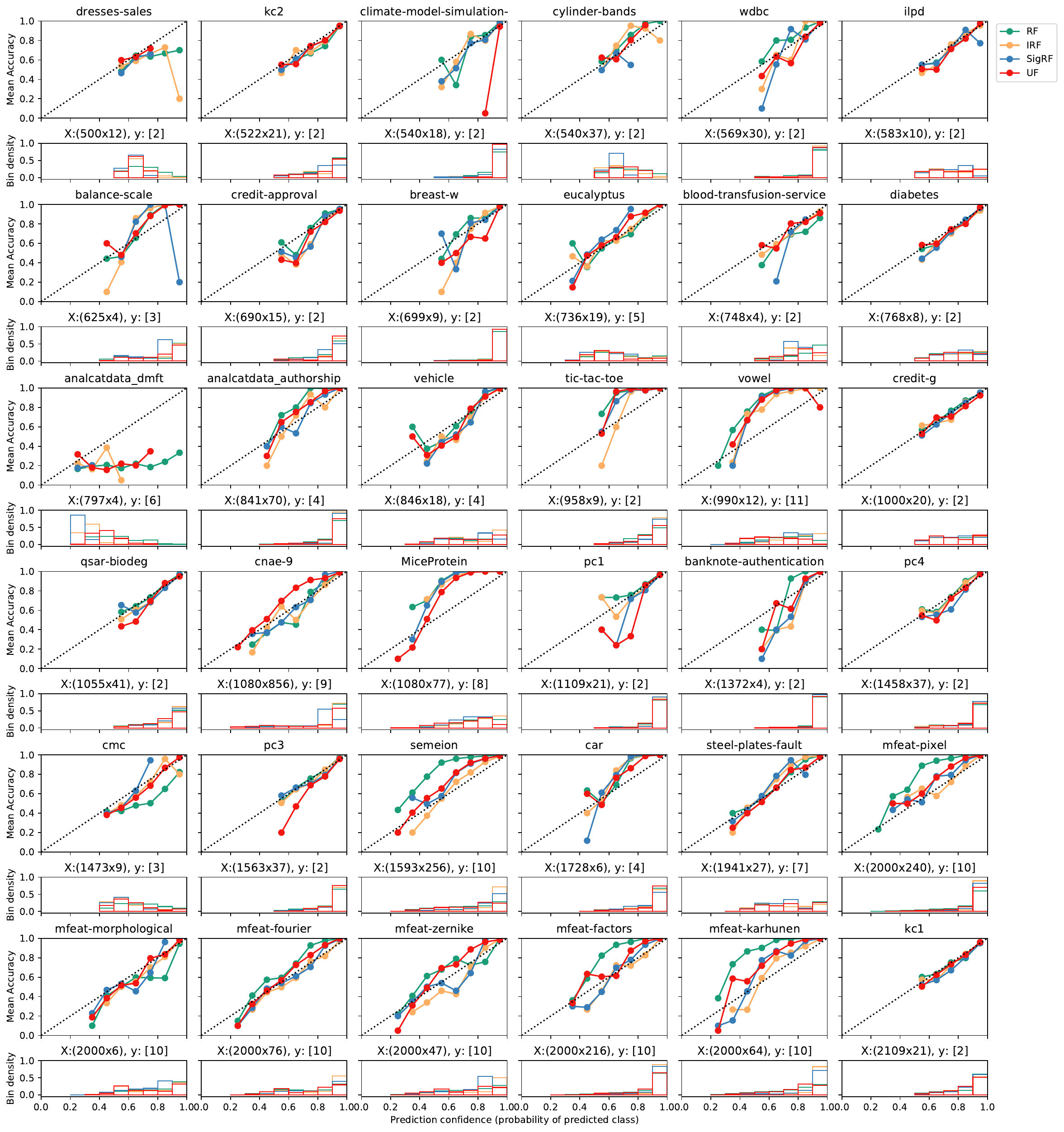}
  \caption{Each plot (one per the half of CC18 datasets) shows the average accuracy of samples in each of 10 equally spaced bins on $(0, 1]$ per the posterior of the predicted class. An ECE of 0 corresponds to all points occurring on the dashed line $y = x$. The proportion of samples in each bin is plotted beneath.}
  \label{fig:supp_ece_density_1}
\end{figure}

\begin{figure}[!hb]
  \centering
  \includegraphics[width=\linewidth]{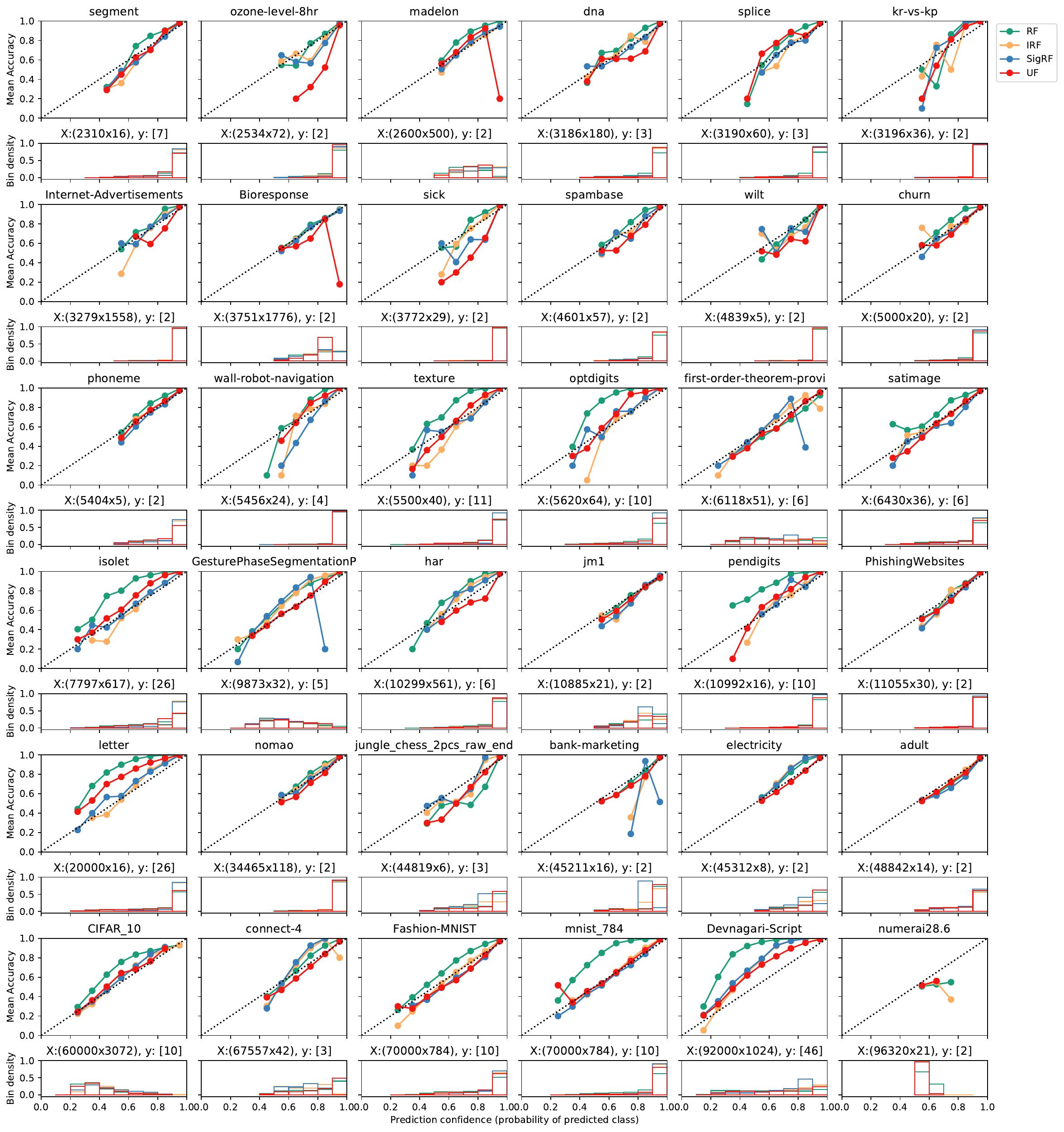}
  \caption{Each plot (one per the second half of CC18 datasets) shows the average accuracy of samples in each of 10 equally spaced bins on $(0, 1]$ per the posterior of the predicted class. An ECE of 0 corresponds to all points occurring on the dashed line $y = x$. The proportion of samples in each bin is plotted beneath.}
  \label{fig:supp_ece_density_2}
\end{figure}

\end{document}


\title{Supplementary Material: \\ Estimating Information-Theoretic Quantities with Uncertainty Forest}
\author{Ronak Mehta\thanks{Department of Biomedical Engineering, Johns Hopkins University at Baltimore, MD 21218. \texttt{jovo@jh.edu}}
\and Richard Guo\thanks{Department of Computer Science, Johns Hopkins University at Baltimore, MD 21218.}
\and Jes\'us Arroyo\thanks{Center for Imaging Science, Johns Hopkins University, Baltimore, MD 21218.}
\and Mike Powell\footnotemark[1]
\and Hayden Helm\footnotemark[3]
\and Cencheng Shen\thanks{Department of Applied Economics and Statistics, University of Delaware, Newark, DE 19716.}
\and Joshua T. Vogelstein\footnotemark[1]
}

\date{}
\maketitle

\section{Proofs}
\label{sec:proof}
\setcounter{theorem}{0}
\setcounter{assumption}{0}

In this section, we present consistency results regarding estimation of conditional entropy and mutual information via Uncertainty Forest. The argument follows as a nearly direct consequence of \cite{Athey2019-uw}, in which random forests that are grown according to some specifications, and solve locally weighted estimating equations are consistent and asymptotically Gaussian. We review the assumptions.
\begin{assumption}
    Suppose that $X$ is supported on $\Real^d$ and has a density which non-zero almost everywhere. This is equivalent to have $X$ be uniformly distributed in $[0, 1]^d$ due to the monotone invariance of trees \cite{denil14}.
\end{assumption}
\begin{assumption}
    Suppose that for each $y \in [K]$ the conditional probability $p(y \mid x)$ is Lipshitz continuous on $\Real^d$.
\end{assumption}
\begin{assumption}
    Each tree $b$ is constructed such that for all $x \in \mathcal{X}, \epsilon > 0$, $P[N_b(x) < \epsilon] \rightarrow 0$ as $n \rightarrow \infty$.
    \label{asm:bigcells}
\end{assumption}

Next, we summarize the notation and main result of \cite{Athey2019-uw}. Let $\theta(x)$ be a parameter that is implicitly defined as the solution to some equation $M_\theta(x) = 0$. For example, in estimating the conditional mean $\theta(x) = \E[Y \mid X = x]$, we can use
\begin{align*}
    M_\theta(x) = \E[Y \mid X = x] - \theta(x) = 0.
\end{align*}
To simplify notation, we suppress the dependency of $x$ in $\theta$ and simply write $\theta=\theta(x)$ wherever this dependency can be deduced from the context. To develop a sample estimate for $\theta$, we start with a sample score function $\psi_\theta$ such that
\begin{align*}
    \E[\psi_\theta(Y) \mid X = x] = M_\theta(x).
\end{align*}
Now, let us be given a dataset $\{(X_1, Y_1), ..., (X_n, Y_n)\}$. The sample estimate $\hat{\theta}$ solves a locally weighted estimating equation
\begin{align*}
    \sum_{i=1}^n \alpha_i(x) \psi_\theta(Y_i) = 0.
\end{align*}
In the case of conditional mean estimation, we will have $\psi_\theta(Y) = Y - \theta$. The $\alpha_i(x)$'s  weigh highly observations with $X_i$ close to $x$, as to approximate the conditional expectation of $\psi_\theta(Y)$, or $M_\theta(x)$. In a random forest, these weights amount to being the empirical probability that the test point $x$ shares a leaf with each training point $X_i$. Under mild assumptions, \cite{Athey2019-uw} have shown such a method to be consistent for the estimation of $\theta(x)$. For such a result to be used for our purpose of estimating conditional entropy, we define $M_\theta(x)$ as 
\begin{align*}
    M_\theta(x) &= p(y \mid x) - \theta
\end{align*}
Using this definition, we can derive consistency of $\hat{H}(Y \mid X = x)$. Given this function of $x$, we have the conditional entropy written as
\begin{align*}
    H(Y \mid X) &= \E_{X'}[H(Y \mid X = X')]
\end{align*}

We first confirm that the finite-sample corrected posterior $\hat{p}_b(y \mid x)$ approaches the uncorrected $\tilde{p}_b(y \mid x)$ for large $n$. 
\begin{lemma} 
\label{thm:finite}
As $n \rightarrow \infty$
\begin{align*}
    |\hat{p}(y \mid x) - \tilde{p}(y \mid x)| & \overset{P}{\rightarrow} 0
\end{align*}
\end{lemma}
\begin{proof}
    We index the estimates with $n$ to make explicit the dependence on sample size. For tree $b$, in the cases where $0 < \tilde{p}_{b,n}(y \mid x) < 1$ for all $y$, we have $|\hat{p}_{b,n}(y \mid x) - \tilde{p}_{b,n}(y \mid x)| = 0$. Otherwise, for $\mathcal{N} = \{y : \tilde{p}_{b,n}(y \mid x) = 0\}$ and $y  \in \mathcal{N}$, then
    \begin{align*}
        |\hat{p}_{b,n}(y \mid x) - \tilde{p}_{b,n}(y \mid x)| &= \frac{1}{\kappa N_{b,n}(x)}.
    \end{align*}
    By Assumption \ref{asm:bigcells}, we have that $N_{b,n}(x) \overset{n}{\rightarrow} \infty$ for all $x$. Thus,
    \begin{align*}
        \frac{1}{\kappa N_{b,n}(x)} \overset{P}{\rightarrow} 0
    \end{align*}
    Similarly, the constant $c = \frac{|\mathcal{N}|}{\kappa N_{b,n}(x)} \leq \frac{K}{\kappa N_{b,n}(x)} \rightarrow 0$ in probability. Thus, for $y \not\in$ $\mathcal{N}$,
    \begin{align*}
        |\hat{p}_{b,n}(y \mid x) - &\tilde{p}_{b,n}(y \mid x)| \\
        &= |(1-c)\tilde{p}_{b,n}(y \mid x) - \tilde{p}_{b,n}(y \mid x)|\\
        &= |c \cdot \tilde{p}_{b,n}(y \mid x)| \leq c \rightarrow 0
    \end{align*}
    as $n \rightarrow \infty$. The estimate
    $\hat{p}_n(y \mid x) = \frac{1}{B} \sum_{b=1}^B \hat{p}_{b,n}(y \mid x)$ is a finite average of the $\hat{p}_{b,n}(y \mid x)$. Given $\epsilon > 0$,
    \begin{align*}
        P[&|\hat{p}_n(y \mid x) - \tilde{p}_n(y \mid x)| > \epsilon] 
        \\ &=P[|\tfrac{1}{B}\sum_{b=1}^B \hat{p}_{b,n}(y \mid x) - \tilde{p}_{b,n}(y \mid x)| > \epsilon]\\ 
        &\leq P[\sum_{b=1}^B |\hat{p}_{b,n}(y \mid x) - \tilde{p}_{b,n}(y \mid x)| > B\epsilon]\\ 
        &\leq P\left[\bigcup_{b=1}^B |\hat{p}_{b,n}(y \mid x) - \tilde{p}_{b,n}(y \mid x)| > \epsilon\right]\\
        &\leq \sum_{b=1}^B P[|\hat{p}_{b,n}(y \mid x) - \tilde{p}_{b,n}(y \mid x)| > \epsilon]\\
        &= B \cdot P[|\hat{p}_{1,n}(y \mid x) - \tilde{p}_{1,n}(y \mid x)| > \epsilon]\\
        &= B \cdot \big( P[|\hat{p}_{1,n}(y \mid x) - \tilde{p}_{1,n}(y \mid x)| > \epsilon \\
        &\quad \quad \quad \quad \cap \tilde{p}_{1,n}(y \mid x) = 0]
         \\&+ P[|\hat{p}_{1,n}(y \mid x) - \tilde{p}_{1,n}(y \mid x)| > \epsilon \\
         & \quad \quad \quad \quad \cap \tilde{p}_{1,n}(y \mid x) \neq 0] \big)\\
        &\overset{n}{\rightarrow} 0
    \end{align*}
\end{proof}

Second, we require that honestly estimated posteriors $\tilde{p}_n(y \mid x)$ (without the finite sample correction) converge to the true posteriors. This can be done by application of Theorem 3 from \cite{Athey2019-uw}.

\begin{lemma}
    \label{thm:prob}
    For all $y \in [K]$ and $x \in \Real^d$, $\tilde{p}_n(y \mid x) \overset{P}{\rightarrow} p(y \mid x)$ as $n \rightarrow \infty$.
\end{lemma}
\begin{proof}
    For a fixed $(x, y)$, we can define $\theta$ as the solution to
    \begin{align*}
        M_\theta(x) = p(y \mid x) - \theta = 0.
    \end{align*}
    The sample equivalent is thus
    \begin{align*}
        \psi_\theta(Y) = \I[Y = y] - \theta,
    \end{align*}
    and we have that
    \begin{align*}
        \E[\psi_\theta(Y) \mid X = x] = M_\theta(x).
    \end{align*}
    Because our forest is learned according to Specification 1 in \cite{Athey2019-uw}, we must confirm Assumptions 1 - 6 from Section 3 of that paper to apply their result. Only Assumption 1 is a proper assumption on the distribution, whereas the remaining are confirmed true based on our choice of $M_\theta$ and $\psi_\theta$.
    \begin{enumerate}
        \item For fixed $\theta$, $M_\theta(x) = p(y \mid x) - \theta$ is Lipschitz in $x$. We took this as a standard assumption on the joint distribution of $(X,Y)$.
        \item For fixed $x$, $M_\theta(x)$ is twice continuously differentiable in $\theta$, $\frac{\partial^2 M_\theta}{\partial \theta^2} = 0$ and $\frac{\partial}{\partial \theta}M_\theta = -1$, which is invertible.
        \item The function $\psi_\theta$ satisfies
        \begin{align*}
            \sup_{x \in \mathcal{X}}&(\text{Var}[\psi_\theta(Y) - \psi_\theta(Y)]) 
            \\&= \sup_{x \in \mathcal{X}}(\text{Var}[\I[Y = y] - \theta - (\I[Y = y] - \theta')])
            \\&= \sup_{x \in \mathcal{X}}(\text{Var}[\theta' - \theta'])
            \\&= 0,
        \end{align*}
        and hence $\sup_{x \in \mathcal{X}}(\text{Var}[\psi_\theta(Y) - \psi_{\theta'}(Y)]) \leq L||\theta - \theta'||$ for all $\theta, \theta'$ and some $L \geq 0$.
        \item The function $\psi_\theta$ is itself Lipschitz in $\theta$.
        \item The solution of
        \begin{align*}
            \sum_{i=1} \alpha_i(x) \psi_{\theta}(Y_i) = 0
        \end{align*}
        in $\theta$ exists, and is equal to
        \begin{align*}
            \hat{\theta}(x) = \sum_{i=1}^n \alpha_i(x) \I[Y_i = y]
        \end{align*}
        \item The function $\psi_\theta$ is the negative subgradient of a convex function, and $M_\theta$ is the negative subgradient of a strongly convex function (with respect to $\theta$). Choose
        \begin{align*}
            \bf{\Psi}(\theta) &= \frac{1}{2}(\I[Y = y] - \theta)^2,\\
            \bf{M}(\theta) &= \frac{1}{2}(p(y \mid x) - \theta)^2
        \end{align*}
        as these functions.
    \end{enumerate}
    Granted the above assumptions, by Theorem 3 of \cite{Athey2019-uw} we have for every $x, y$,
    \begin{align*}
        \hat{p}_n(y \mid x) \overset{P}{\rightarrow} p(y \mid x).
    \end{align*}
\end{proof}
 By Lemma \ref{thm:prob} and continuity, we have the consistency of the honest forest estimate of conditional entropy.
 Let $\tilde{H}(Y \mid X = x) = -\sum_{y \in [K]} \tilde{p}_n(y \mid x) \log \tilde{p}_n(y \mid x)$.
\begin{corollary} For each $x\in \Real^d$,
\label{thm:no_finite}
\begin{align*}
    \tilde{H}(Y \mid X = x) \overset{P}{\rightarrow} H(Y \mid X = x).
\end{align*}
\end{corollary}
\begin{proof}
    The function 
    \begin{align*}
        h(p) = \begin{cases}
        0 &\text{ if } p = 0\\
        -p \log p &\text{ otherwise}
        \end{cases}
    \end{align*}
    is continuous on $[0,1]$. Similarly, the finite sum $\sum_{k=1}^K h(p_k)$ is continuous on $\{(p_1, ..., p_K) : 0 \leq p_k \leq 1, \sum_{k=1}^K p_k = 1\}$. Thus, by Lemma \ref{thm:prob} and the continuous mapping theorem, we have the desired result.
\end{proof}

\begin{lemma}
    \label{thm:function}
    The Uncertainty Forest function estimate $\hat{H}(Y \mid X = x)$ converges in probability to $H(Y \mid X = x)$ for each $x \in \Real^d$.
\end{lemma}
\begin{proof}
     By continuity of $\sum_{k=1}^K h(p_k)$ on $\{(p_1, ..., p_K) : 0 \leq p_k \leq 1, \sum_{k=1}^K p_k = 1\}$, and Lemma \ref{thm:finite} we have that for any $x$,
    \begin{align}
        \left|\hat{H}(Y \mid X = x) - \tilde{H}(Y \mid X = x)\right| \overset{P}{\rightarrow} 0
        \label{eqn:entropy_correction}
    \end{align}
    Then, given $\epsilon > 0$,
    \begin{align*}
        P[|\hat{H}(Y &\mid X = x) - H(Y \mid X = x)| > \epsilon] 
        \\&\leq P[|\hat{H}(Y \mid X = x) - \tilde{H}(Y \mid X = x)|\\ 
        &\quad + |\tilde{H}(Y \mid X = x) - H(Y \mid X = x)| > \epsilon]\\
        &\leq P[|\hat{H}(Y \mid X = x) - \tilde{H}(Y \mid X = x)| > \frac{\epsilon}{2}]\\
        &\quad + P[|\tilde{H}(Y \mid X = x) - H(Y \mid X = x)| > \frac{\epsilon}{2}]\\
        &\overset{n}{\rightarrow} 0,
    \end{align*}
    by Lemma \ref{thm:no_finite} and \ref{eqn:entropy_correction}.
\end{proof}
\setcounter{theorem}{0}

For simplicity of notation, let $n$ refer only to the total sample size of the partition $\mathcal{D}^{\text{P}}$ and vote $\mathcal{D}^{\text{V}}$ sets. Let $\nu$ be the size of the evaluation set $\mathcal{D}^{\text{E}}$. If a subset of the labelled data is being used, then this may be some fraction $0 < \gamma < 1$ of the size of the full training data. In this case, we consider any growing evaluation set size $\nu$. 

\begin{theorem}\textsc{(Consistency of the conditional entropy estimate)} Suppose the conditions of the preceding lemmas. Then the conditional entropy estimate is consistent as $n, \nu \rightarrow \infty$, that is,
    \begin{equation*}
        \hat{H}_{n, \nu}(Y \mid X) \overset{P}{\rightarrow} H(Y \mid X).
    \end{equation*}
    \label{thm:cond_ent}
\end{theorem}
\begin{proof}[Proof of Theorem \ref{thm:cond_ent}]
    By Lemma \ref{thm:function} we have the pointwise consistency of $\hat{H}_n(Y \mid X = x)$ (now indexed by $n$ explicitly) for $H(Y \mid X = x)$. We compute the limits with respect to $n$ and $\nu$ in either order to achieve the result. Let $\hat{H}_{n, \nu}(Y \mid X) = \frac{1}{\nu} \sum_{i \in \mathcal{D}^{\text{E}}} \hat{H}_n(Y \mid X = X_i)$.
    \begin{align*}
        \lim_{\nu \rightarrow \infty} \lim_{n \rightarrow \infty} &\hat{H}_{n, \nu}(Y \mid X) \\
        &= \lim_{\nu \rightarrow \infty} \frac{1}{\nu} \sum_{i \in \mathcal{D}^{\text{E}}} \lim_{n \rightarrow \infty} \hat{H}_n(Y \mid X = X_i)\\
        &= \lim_{\nu \rightarrow \infty} \frac{1}{\nu} \sum_{i \in \mathcal{D}^{\text{E}}} H(Y \mid X = X_i)\\
        &= H(Y \mid X)
    \end{align*}
    by the Weak Law of Large Numbers, where the limits are taken in probability. On the other hand,
    \begin{align*}
        \lim_{n \rightarrow \infty} \lim_{\nu \rightarrow \infty} &\hat{H}_{n, \nu}(Y \mid X) \\
        &= \lim_{n \rightarrow \infty} \lim_{\nu \rightarrow \infty} \frac{1}{\nu} \sum_{i \in \mathcal{D}^{\text{E}}} \hat{H}_n(Y \mid X = X_i)\\
        &= \lim_{n \rightarrow \infty} \E_{X'}[\hat{H}_n(Y \mid X = X')]\\
        &= \lim_{n \rightarrow \infty} \int_{x \in \Real^d} \hat{H}_n(Y \mid X = x) \ dF_X,
    \end{align*}
    also by the Weak Law. Because of the finiteness of $[K]$, the entropy-like term $|\hat{H}_n(Y \mid X = x)|$ is bounded by $\log K$ for all $n$ and $x$. Therefore, by the Dominated Convergence Theorem,
    \begin{align*}
        \lim_{n \rightarrow \infty} \int_{x \in \Real^d} &\hat{H}_n(Y \mid X = x) \ dF_X \\
        &= \int_{x \in \Real^d} \lim_{n \rightarrow \infty}  \hat{H}_n(Y \mid X = x) \ dF_X\\
        &= \int_{x \in \Real^d} H(Y \mid X = x) \ dF_X\\
        &= \E_{X'}[H(Y \mid X = X')]\\
        &= H(Y \mid X)
    \end{align*}
    Thus, $\hat{H}(Y \mid X)$ is consistent for $H(Y \mid X)$.
\end{proof}
Letting the $n$ and $\nu$ be a fixed fraction of the training set size is a special case of both growing to infinity.

The second result concerns the mutual information estimate. The entropy $H(Y)$ is estimated in the natural way, as in
\begin{align*}
    \hat{H}(Y) &= -\sum_{y \in [K]} \hat{p}(y) \log \hat{p}(y)
\end{align*}
where $\hat{p}(y) = \frac{1}{n}\sum_{i=1}^n \I\{Y_i = y\}$ (and set to zero when appropriate). Consequently, the mutual information is estimated as
\begin{align*}
    \hat{I}_{n, \nu}(X,Y) &= \hat{H}(Y) - \hat{H}(Y \mid X).
\end{align*}
Consider the same assumptions as Theorem \ref{thm:cond_ent}.
\begin{theorem}\textsc{(Consistency of the mutual information estimate)}
    $\hat{I}_{n, \nu}(X,Y)\rightarrow_p I(X; Y)$ as $n, \nu \rightarrow \infty$.
\end{theorem}
\begin{proof}
    For $i.i.d.$ observations of $Y_i$, it is clear that $\hat{H}(Y)$ is a consistent estimate for $H(Y)$. By Theorem \ref{thm:cond_ent} we have the consistency of $\hat{H}(Y \mid X)$ for $H(Y \mid X)$. Thus, the consistency of $\hat{I}(X,Y)$ follows immediately.
\end{proof}

\section{Pseudocode}
\label{sec:code}
The \textproc{FitRandomForest}, \textproc{GetInBagSamples}, and \textproc{ApplyTree} operations, as well as the \textproc{NumClasses} and \textproc{NumLeaves} fields are all standard functions in the \texttt{scikit-learn} decision tree and bagging classifier modules \cite{scikit-learn}. 
\begin{algorithm}
   \caption{Uncertainty Forest (Forest-level)}
   \label{alg:cerfestimate}
\begin{algorithmic} 
\Require Training set $\mathcal{D}_n$ and evaluation set $\mathcal{D}^{\text{E}}$, $\theta= \{$number of trees $B$, minimum leaf size $k$, subsample size $s$ $\}$, finite correction constant $\kappa$.
\Ensure Conditional entropy estimate $\hat{H}(Y \mid X)$.
\Function{UncertaintyForest}{$\mathcal{D}_n$, $\mathcal{D}^{\text{E}}$, $\theta$, $\kappa$}
    \State $\mc{T} = \mathcal{D}_n - \mathcal{D}^{\text{E}}$
    \State model = \textproc{FitRandomForest}$(\mc{T}, \theta)$.
    \State posteriors = [].
    \For{tree $b$ \textbf{in} model}
        \State posterior = 
        \textproc{EstimatePosterior}($\mc{T}$, model,\\
        \quad \quad \quad \quad
        $\kappa$, $b$)
        \State posteriors.append(posterior)
    \EndFor
    \State $\hat{H}(Y \mid X)$ = \textproc{EvaluatePosteriors}($\mathcal{D}^{\text{E}}$, model, \\
    \quad \quad \quad posteriors).
    \State \Return $\hat{H}(Y \mid X)$.
\EndFunction
\end{algorithmic}
\end{algorithm}

\begin{algorithm}
   \caption{Uncertainty Forest (Tree-level)}
   \label{alg:cerfestimate2}
\begin{algorithmic} 
\Require Training set $\mathcal{D}_n$ and evaluation set $\mathcal{D}^{\text{E}}$, $\theta= \{$number of trees $B$, minimum leaf size $k$, subsample size $s$ $\}$, finite correction constant $\kappa$.
\Ensure Conditional entropy estimate $\hat{H}(Y \mid X)$.
\Function{UncertaintyForest}{$\mathcal{D}_n$, $\mathcal{D}^{\text{E}}$, $\theta$, $\kappa$}
    \State $\mc{T} = \mathcal{D}_n - \mathcal{D}^{\text{E}}$
    \State model = \textproc{FitRandomForest}$(\mc{T}, \theta)$.
    \State conditional\_entropies = []
    \For{tree $b$ \textbf{in} model}
        \State posterior = \textproc{EstimatePosterior}($\mc{T}$, model, \\ \quad \quad \quad \quad $\kappa$, $b$)
        \State conditional\_entropies.append(\\
        \quad \quad \quad \quad \textproc{EvaluatePosteriors}($\mathcal{D}^{\text{E}}$, [model[$b$]], \\
        \quad \quad \quad \quad[posterior])).
    \EndFor
    \State $\hat{H}(Y \mid X)$ = conditional\_entropies.\textproc{Mean}().
    \State \Return $\hat{H}(Y \mid X)$.
\EndFunction
\end{algorithmic}
\end{algorithm}

\begin{algorithm}
   \caption{Posterior Estimation}
   \label{alg:est_posterior}
\begin{algorithmic} 
\Require training set $\mc{T}$ (full dataset without evaluation), fitted random forest model, finite correction constant $\kappa$, tree index $b$
\Ensure Posterior probability estimate tree $b$.
\Function{EstimatePosterior}{$\mc{T}$, model, $\kappa$, $b$}
    \State $K$ = model.\textproc{NumClasses}.
    \State $\mathcal{D}_\text{P}$ = \textproc{GetInBagSamples}(model, $b$)
    \State $\mathcal{D}_\text{V} = \mc{T} - \mathcal{D}_\text{P}$.
    \State $L$ = model[$b$].\textproc{NumLeaves}.
    \State vote\_counts = $[0]^{L \times K}$.
    \For{observation $(x, y)$ \textbf{in} $\mathcal{D}_\text{V}$}
        \State $l$ = model[$b$].\textproc{ApplyTree}($x$)
        \State vote\_counts[$l$, $y$] = vote\_counts[$l$, $y$] + 1.
    \EndFor
    \State leaf\_sizes = \textproc{RowSum}(vote\_counts).
    \State posterior = $[0]^{L \times K}$.
    \For{leaf index $l$ \textbf{in} $[L]$}
        \For{class $y$ \textbf{in} $[K]$}
            \State posterior[$l$, $y$] = vote\_counts[$l$, $y$] / \\
            \quad \quad \quad \quad \quad leaf\_sizes[$l$].
        \EndFor
    \EndFor
    \State posterior = \textproc{FiniteSampleCorrect}(posterior, \\
    \quad \quad $\kappa$, leaf\_sizes)
    \State \Return posterior
\EndFunction
\end{algorithmic}
\end{algorithm}

\begin{algorithm}
   \caption{Posterior Evaluation}
   \label{alg:eval_posterior}
\begin{algorithmic} 
\Require evaluation set $\mathcal{D}^{\text{E}}$, fitted random forest model, posteriors for each tree
\Ensure Conditional entropy estimate $\hat{H}(Y \mid X)$.
\Function{EvaluatePosteriors}{$\mathcal{D}^{\text{E}}$, model, posteriors}
    \State $B$ = model.\textproc{NumTrees}.
    \For{observation $x_i$ \textbf{in} $\mathcal{D}^{\text{E}}$}
        \For{tree $b$ \textbf{in} model}
            \State $l$ = model[$b$].\textproc{ApplyTree}($x$).
            \For{class $y$ \textbf{in} $[K]$}
                \State posterior = posteriors[$b$].
                \State $p_b(y \mid x_i)$ = posterior[$l$, $y$].
            \EndFor
        \EndFor
        \State $p_i(y \mid x)  = \frac{1}{B} \sum_{b=1}^B p_b(y \mid x_i)$.
        \State $\hat{H}_i(Y \mid X) = -\sum_{y \in [K]}p_i(y \mid x) \log p_i(y \mid x)$.
    \EndFor
    \State $\hat{H}(Y \mid X) = \frac{1}{|\mathcal{D}^{\text{E}}|}\sum_{i \in \mathcal{D}^{\text{E}}} \hat{H}_i(Y \mid X)$.
    \State \Return $\hat{H}(Y \mid X)$.
\EndFunction
\end{algorithmic}
\end{algorithm}

\begin{algorithm}
   \caption{Finite Sample Correction}
   \label{alg:finite}
\begin{algorithmic} 
\Require posterior, finite correction constant $\kappa$, leaf\_sizes
\Ensure Corrected posterior.
\Function{FiniteSampleCorrect}{posterior, $\kappa$, leaf\_sizes}
    \State $(L, K)$ = posterior.\textproc{Shape}.
    \For{node index $l$ \textbf{in} $[L]$}
        \For{class $y$ \textbf{in} $[K]$}
            \If{posterior[$l$, $y$] == 0} 
                \State posterior[$l$, $y$] = $\frac{1}{\kappa \cdot \text{leaf\_sizes}[l]}$
            \EndIf
        \EndFor
    \EndFor
    \State normalizing\_factor = \textproc{RowSum}(posterior).
    \For{node index $l$ \textbf{in} $[L]$}
        \For{class $y$ \textbf{in} $[K]$}
            \State posterior[$l$, $y$] = posterior[$l$, $y$] / normalizing\_factor[$l$]
        \EndFor
    \EndFor
    \State \Return posterior.
\EndFunction
\end{algorithmic}
\end{algorithm}

Note that $ \mathcal{D}^{\text{E}}$ can contain unlabelled data, as only the $x_i$ values are used.

\newpage

\bibliography{siam_bib}
\bibliographystyle{plain}